\documentclass{article} 
\usepackage{iclr2027_conference,times}

\usepackage{amsmath,amsfonts,bm}

\def\eqref#1{equation~\ref{#1}}

\def\1{\bm{1}}

\DeclareMathAlphabet{\mathsfit}{\encodingdefault}{\sfdefault}{m}{sl}
\SetMathAlphabet{\mathsfit}{bold}{\encodingdefault}{\sfdefault}{bx}{n}

\usepackage{url}
\usepackage{graphicx}
\usepackage{booktabs}
\usepackage{placeins}
\usepackage{capt-of}
\usepackage{needspace}
\usepackage{multirow}
\usepackage{microtype}

\usepackage{amsmath,amssymb}
\usepackage{algorithm}
\usepackage{algpseudocode}
\usepackage{wrapfig}
\restylefloat{algorithm}
\usepackage{natbib}

\newsavebox{\meshEqBox}
\newcommand{\meshEqPair}[2]{%
  \sbox{\meshEqBox}{$\displaystyle #1\qquad #2$}%
  \ifdim\dimexpr\wd\meshEqBox+3em\relax<\linewidth
    #1\qquad #2%
  \else
    \begin{gathered}#1\\#2\end{gathered}%
  \fi
}

\title{MESH-Harness: Self-Improving Agent \\
Harnesses via Bandit-Guided \\
Compositional Evolution}

\author{%
\parbox{\dimexpr\textwidth-2\tabcolsep\relax}{\raggedright\fontsize{9}{10}\selectfont
\textbf{Zhiwei Shang$^{1,6}$, Yu Huo$^{2,6}$, Mingrong Gong$^{3}$, An Yan$^{4,6}$, Zikun Qu$^{1}$,}\\
\textbf{Junhao Dong$^{3}$, Bryan Kian Hsiang Low$^{5}$, Chenglin Wu$^{6}$, Zhongxiang Dai$^{1}$}\\
\normalfont\fontsize{8}{9}\selectfont
$^{1}$The Chinese University of Hong Kong, Shenzhen\\
$^{2}$The Chinese University of Hong Kong\\
$^{3}$Nanyang Technological University\\
$^{4}$Fudan University\\
$^{5}$National University of Singapore\\
$^{6}$DeepWisdom\\
Correspondence to: Zhongxiang Dai \texttt{<daizhongxiang@cuhk.edu.cn>}
}}

\iclrfinalcopy
\makeatletter
\renewcommand{\@maketitle}{%
  \vspace*{-42pt}%
  \vbox{\hsize\textwidth
    {\LARGE\scshape\@title\par}%
    \vskip 7pt
    \noindent\hspace*{\tabcolsep}\@author\par
    \vskip 2.56pt}}
\makeatother
\usepackage{amsthm}
\usepackage{array}
\usepackage{tabularx}
\usepackage{longtable}
\usepackage{xcolor}
\usepackage{listings}
\usepackage[most]{tcolorbox}
\newtheorem{meshproposition}{Proposition}[section]

\definecolor{meshpanelhead}{gray}{0.25}
\definecolor{meshpanelbody}{gray}{0.96}
\lstdefinestyle{meshsupp}{
  language=Python,
  basicstyle=\ttfamily\fontsize{8.7}{10.4}\selectfont,
  keywordstyle=\bfseries,
  commentstyle=\itshape,
  stringstyle=\ttfamily,
  columns=fullflexible,
  keepspaces=true,
  showstringspaces=false,
  breaklines=true,
  breakatwhitespace=true,
  tabsize=4,
  aboveskip=0pt,
  belowskip=0pt
}
\newtcblisting{meshpanel}[1]{
  enhanced,
  listing only,
  listing options={style=meshsupp},
  title={#1},
  colback=meshpanelbody,
  colframe=meshpanelhead,
  colbacktitle=meshpanelhead,
  coltitle=white,
  fonttitle=\bfseries\normalsize,
  boxrule=0.7pt,
  arc=2mm,
  left=2.5mm,right=2.5mm,top=2mm,bottom=2mm,
  before skip=10pt,after skip=7pt
}

\usepackage{hyperref}
\hypersetup{
  pdftitle={MESH-Harness: Self-Improving Agent Harnesses via Bandit-Guided Compositional Evolution},
  pdfauthor={Zhiwei Shang, Yu Huo, Mingrong Gong, An Yan, Zikun Qu, Junhao Dong, Bryan Kian Hsiang Low, Chenglin Wu, Zhongxiang Dai}
}
\begin{document}

\maketitle
\suppressfloats[t]

\begingroup
\raggedbottom
\setcounter{topnumber}{1}
\setlength{\parskip}{2pt plus 1pt minus 0.5pt}
\setlength{\textfloatsep}{6pt plus 2pt minus 2pt}
\setlength{\floatsep}{8pt plus 2pt minus 2pt}
\setlength{\intextsep}{6pt plus 2pt minus 2pt}
\setlength{\abovecaptionskip}{3pt}
\setlength{\abovedisplayskip}{5pt plus 2pt minus 2pt}
\setlength{\belowdisplayskip}{5pt plus 2pt minus 2pt}
\setlength{\abovedisplayshortskip}{3pt plus 1pt}
\setlength{\belowdisplayshortskip}{4pt plus 1pt minus 1pt}

\begin{abstract}
An \emph{agent harness} is the code that organizes context, maintains state, and coordinates tool calls for a language model. We study how to improve the harness under a limited evaluation budget while keeping model weights fixed. Our method, \textbf{MESH-Harness}, \textcolor{black}{organizes each harness into functional modules} with explicit \textcolor{black}{role-specific} interfaces, allowing \textcolor{black}{alternative implementations of each module to be substituted and recombined}. It uses shared module representations and full-covariance LinUCB to score candidate combinations based on predicted performance and exploration value. Mixed-start coordinate ascent selects complete configurations for evaluation without enumerating the combinatorial space. Validation traces then guide local code edits, and the resulting candidates are incorporated into fixed-capacity \textcolor{black}{role-specific pools} for subsequent recombination. On text tasks, retrieval-augmented mathematical reasoning, code generation, and interactive scientific tasks, MESH-Harness outperforms Meta-Harness by 5.70, 7.01, 2.00, and 5.00 points, respectively, under matched candidate-evaluation budgets. Iterative harness optimization improves MESH-Harness by 5.63--7.79 points over its first-round configurations. For the reported configurations, aggregate test-time cost is 44.2\% lower than that of Meta-Harness, while total cost including search is 14.6\% lower.
These results show that combining module-level design reuse with feedback-driven compositional search can systematically improve agent harnesses while keeping overall optimization cost under control.
\end{abstract}

\section{Introduction}
\label{sec:introduction}

Large language models (LLMs) can act as agents by retrieving information, calling tools, and interacting with environments \citep{yao2023react}. Their behavior also depends on the external program that controls memory, context construction, tool use, and failure recovery. This program is referred to as the \emph{agent harness}. Unlike a single prompt, a harness specifies how successive model calls work together to solve a task. For example, harnesses for long-running agents support this coordination through explicit state handoffs, progress records, and execution checks \citep{young2025harnesses}. Because the harness governs how model calls work together, optimizing it offers a way to improve an agent without updating model weights. AutoHarness synthesizes action validators and policy code through iterative refinement using environment feedback \citep{lou2026autoharness}. Meta-Harness searches over task-specific Python programs, allowing a proposer agent to inspect previous candidates' code, scores, and execution traces \citep{lee2026metaharness}. Both approaches make the surrounding program an object of optimization rather than a fixed implementation.

\begin{figure*}[t]
\centering
\includegraphics[width=1.0\textwidth]{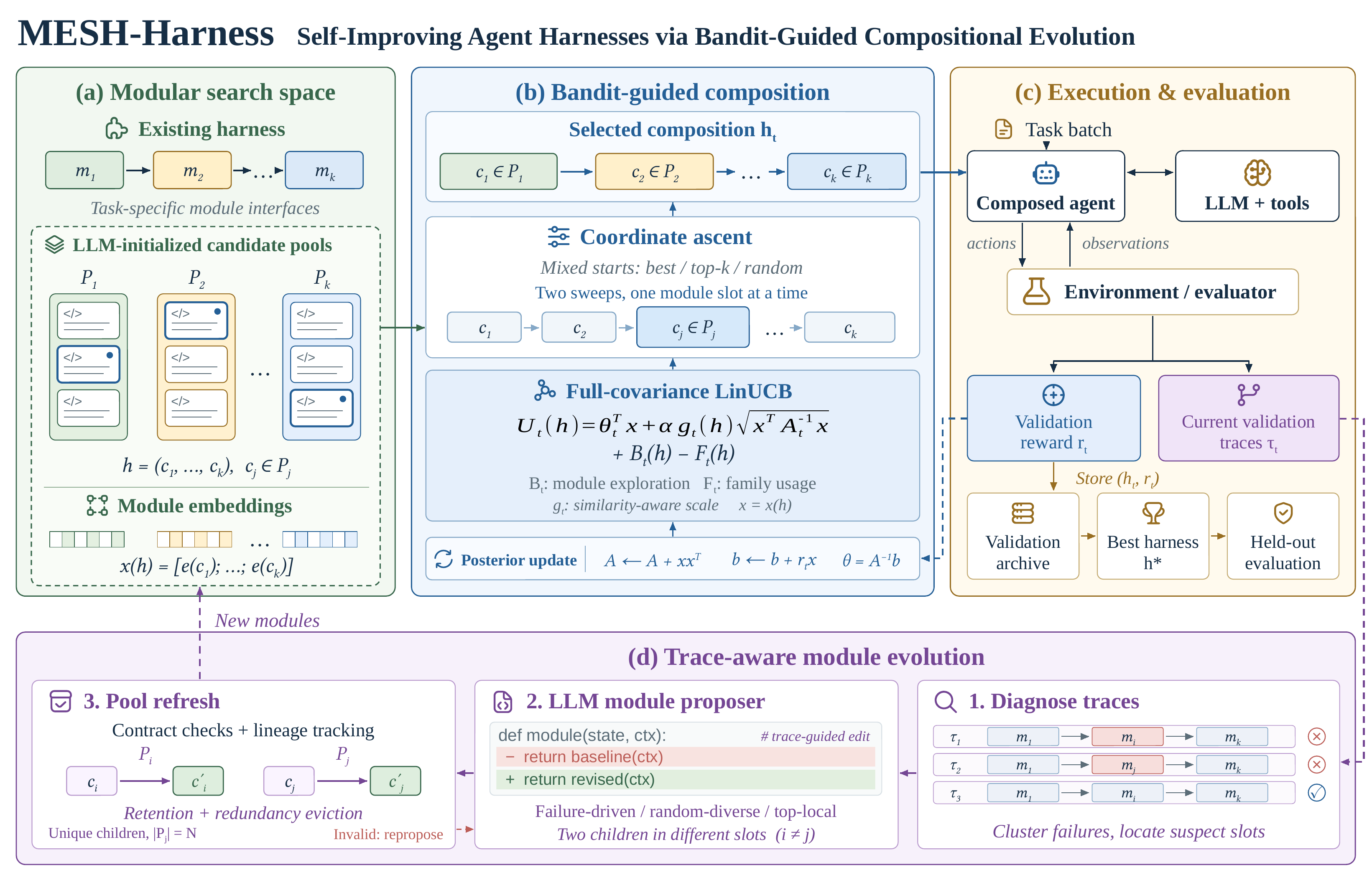}
\caption{\textbf{Overview of MESH-Harness.}
(a) \textcolor{black}{Each functional slot has a fixed interface and a pool of LLM-initialized implementations with cached embeddings}.
(b) Full-covariance LinUCB scores \textcolor{black}{configurations with one implementation per slot}, and mixed-start coordinate ascent selects a configuration for evaluation.
(c) The selected harness is executed to obtain validation rewards and traces.
(d) Trace diagnosis guides \textcolor{black}{edits to implementations} and fixed-capacity pool updates.
Rewards update the bandit model\textcolor{black}{; new implementations join their respective pools for subsequent composition}.}
\label{fig:mesh_overview}
\end{figure*}

The difficulty is that evaluating a harness can be much more expensive than proposing it. Even a small edit may require an end-to-end validation run involving many model calls and environment interactions. Meta-Harness likewise identifies candidate evaluation as its main computational bottleneck \citep{lee2026metaharness}. Under a limited budget, the search must therefore decide not only what to generate, but what to execute. We study how to use past evaluations to make these decisions and obtain larger improvements within a fixed number of candidate evaluations.

We introduce \textbf{MESH-Harness}, a bandit-guided framework for compositional harness evolution. Contextual bandits use shared features to generalize from observed rewards and balance predicted performance against uncertainty \citep{li2010contextual,chu2011linear}. To apply this idea to harnesses, we \textcolor{black}{organize each harness into functional slots with fixed interfaces, maintain a candidate pool per slot, and select one implementation per slot}. \textcolor{black}{Each implementation is encoded separately; concatenating the selected embeddings in slot order represents a complete harness}. \textcolor{black}{A shared surrogate learns from observed harness scores to predict the performance of untested combinations of module implementations}.

Full-covariance LinUCB combines these predictions with an uncertainty bonus to score module combinations. We approximately maximize this criterion through mixed-start coordinate ascent, drawing one start per round from the incumbent, top-$K$ archived configurations, or a random combination and \textcolor{black}{substituting implementations within one slot at a time} rather than enumerating the Cartesian product of candidate pools. Selection uses only cached representations and surrogate scores; the chosen complete harness is then executed on the validation tasks. This separates inexpensive candidate comparison from costly task evaluation.

\begin{figure}[t]
\centering
\includegraphics[width=\textwidth]{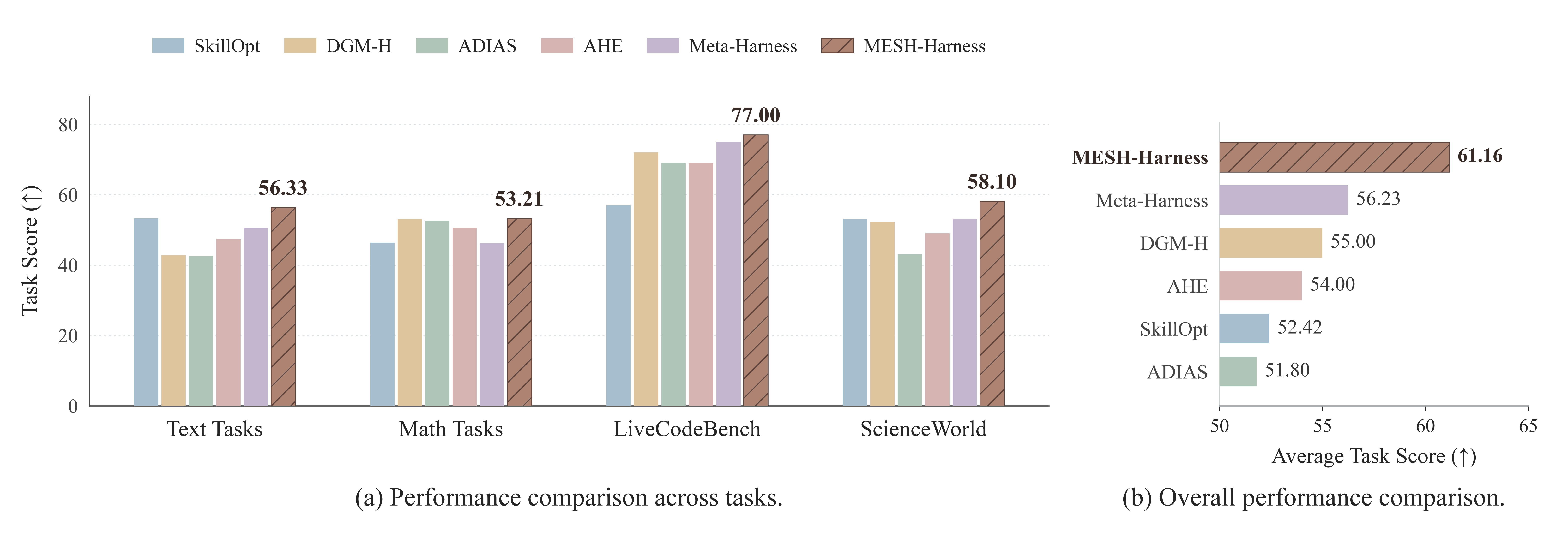}
\caption{\textbf{Test performance across four task domains.}
(a) Text macro-average, overall Math pass@1, LiveCodeBench solve rate under the fixed cap-8 protocol, and ScienceWorld mean progress score.
(b) Unweighted mean of the four task scores.
All scores are on a 0--100 scale, with higher values indicating better performance.
MESH-Harness achieves the highest score on all four aggregate metrics.}
\label{fig:overall_results}
\end{figure}

Selection alone cannot discover implementations absent from the candidate pools. \textcolor{black}{MASPOB selects prompt combinations using bandit-guided coordinate ascent but does not update its candidate pools} \citep{hong2026maspob}. MESH-Harness \textcolor{black}{instead} updates the available \textcolor{black}{implementations in each slot} during search. Validation traces provide evidence of local failures, such as missing information or repeated unsuccessful actions, and guide an LLM in \textcolor{black}{editing implementations while preserving their slot interfaces}. New implementations enter \textcolor{black}{their slot's fixed-capacity pool for combination with candidates in other slots}. Each execution thus supplies two kinds of feedback: scores update the selection model, while traces guide module evolution. A new \textcolor{black}{implementation} need not trigger an immediate end-to-end evaluation.

We evaluate MESH-Harness on text tasks, retrieval-augmented mathematical reasoning, LiveCodeBench \citep{jain2025livecodebench}, and ScienceWorld \citep{wang2022scienceworld}. Under matched candidate-evaluation budgets, it improves final test scores over Meta-Harness by 5.70, 7.01, 2.00, and 5.00 points, respectively. MESH-Harness improves by 5.63--7.79 points over its first-round configurations. On ScienceWorld, it starts at 51.90, below the ReAct reference score of 55.50, and reaches 58.10. Across all four tasks, the reported configurations reduce aggregate test-time cost by 44.2\% relative to Meta-Harness. Total cost, including search and testing the reported configurations, is 14.6\% lower.

Our contributions are:
\begin{itemize}
    \item A bandit-guided harness optimizer that allocates a limited evaluation budget using shared performance predictions and uncertainty estimates.
    \item A module-level representation and mixed-start coordinate ascent procedure that reuse evaluation feedback across configurations without enumerating the full combinatorial space.
    \item A trace-driven module evolution mechanism, evaluated on four tasks in terms of final performance, search gains, and search and test costs.
\end{itemize}

\section{Related Work}
\label{sec:related_work}

\paragraph{Improving agents through external components.}
Recent work improves agents by optimizing the external artifacts that govern their behavior, ranging from natural-language skill documents to executable harnesses.
ACE~\citep{zhang2025ace} incrementally curates execution experience into context playbooks, while SkillOpt~\citep{yang2026skillopt} refines reusable skill documents through trajectory-driven edits and validation-gated updates.
MCE~\citep{ye2026mce} co-evolves context artifacts and the context-engineering skills that produce them.
At the code level, Meta-Harness~\citep{lee2026metaharness} gives an agentic proposer access to historical programs, scores, and traces to guide harness revision.
ADIAS~\citep{jiang2026adias} maintains persistent issue states to direct full-code modifications, and DGM-H~\citep{zhang2026hyperagents} makes both the task agent and its self-improvement procedure editable.
AEvo~\citep{zhang2026aevo} uses accumulated evolution evidence to guide a meta-agent in editing the search procedure or agent context that governs subsequent evolution.
Closely related to our work, AHE~\citep{lin2026ahe} exposes harness components as editable files and uses structured trajectory evidence and edit-outcome checks to guide revision and rollback.
Whereas AHE organizes optimization around successive revisions to a working harness, MESH-Harness maintains alternative implementations in \textcolor{black}{per-slot} pools and learns which complete compositions to evaluate.
Implementations discovered in different iterations can thus be \textcolor{black}{combined by selecting one per slot}, without requiring another LLM rewrite.

\paragraph{Bandit-guided search under limited evaluation budgets.}
Feature-based bandits~\citep{li2010contextual,chu2011linear} provide a way to allocate evaluations by sharing reward information across candidates and balancing exploitation with uncertainty-driven exploration.
TRIPLE~\citep{shi2024triple} casts prompt selection as fixed-budget best-arm identification.
For multi-agent systems, MASPOB~\citep{hong2026maspob} combines a topology-aware GNN surrogate with UCB-guided coordinate search to select prompt combinations from given candidate domains.
MESH-Harness applies uncertainty-guided selection to \textcolor{black}{configurations with one executable implementation per functional slot}.
A shared reward surrogate over component representations allows past evaluations to inform the ranking of untested compositions, while coordinate search avoids enumerating the full combination space.
Unlike search over an unchanged candidate domain, this selection process is interleaved with trace-driven component evolution, so new implementations can enter the pools as optimization proceeds.

Together, these choices separate \textcolor{black}{implementation} generation from evaluation allocation.
MESH-Harness reuses existing implementations without additional LLM rewriting, shares feedback across candidate compositions, and introduces new \textcolor{black}{implementations for existing slots} from execution evidence.
This allows search to exploit prior discoveries without remaining confined to the initial candidate pools.

\section{Preliminaries}
\label{sec:preliminaries}

With a fixed execution model $M$ and tool set, a \emph{harness}
$h$ is a program that organizes model calls, maintains state,
and controls interaction with the environment. On an instance
$z\sim\mathcal D$, it produces a trajectory
$\tau\sim\operatorname{Run}(M,h,z)$ with task score $R(z,\tau)$.
Harness optimization seeks to improve expected performance
without changing model weights:
\begin{equation}
    J(h)=\mathbb E_{\textcolor{black}{z\sim\mathcal D,\;\tau\sim\operatorname{Run}(M,h,z)}}\left[R(z,\tau)\right].
    \label{eq:objective}
\end{equation}

We organize the harness into $k$ functional slots with \textcolor{black}{fixed} interfaces
$\mathcal I=\{I_1,\ldots,I_k\}$. A slot specifies a \textcolor{black}{fixed functional}
responsibility, such as retrieval or failure recovery, rather
than a single model call; execution may include branches and
loops. \textcolor{black}{Before round $t$, pool $P_{j,t}$ contains candidate implementations
satisfying $I_j$. Substitution occurs within a slot; each configuration
selects one implementation per slot}:
\begin{equation}
    \meshEqPair
    {h=(c_1,\ldots,c_k),\quad c_j\in P_{j,t},}
    {\mathcal H_t=\prod_{j=1}^{k}P_{j,t}.}
    \label{eq:search-space}
\end{equation}
The executable harness is
$\operatorname{Compose}_{\mathcal I}(h)$;
$N$ candidates per slot yield $N^k$ configurations.

Each round evaluates a complete configuration $h_t$, yielding
a validation score $r_t$, traces $\mathcal T_t$, and resource-use
records $\kappa_t$, all stored in an archive $\mathcal A_t$.
The score is aggregated under the task's validation protocol
and is an observed estimate rather than the unknown $J(h_t)$.
A bandit pull selects a complete harness: no separate module
rewards are observed. We seek a high-performing configuration
within $T$ evaluations and select the final harness from the
validation archive. Test data are not used for generation,
selection, or stopping.

\section{Method}
\label{sec:method}

\begin{wrapalgorithm}{r}{0.58\textwidth}
\caption{MESH-Harness}
\label{alg:mesh-harness}
\footnotesize
\renewcommand{\algorithmicindent}{1em}
\begin{algorithmic}
\Require $M,\mathcal I,N,T,S,L$; validation evaluator
\State Initialize checked \textcolor{black}{per-slot} pools $\mathcal P=\{P_j\}$ and embeddings
\State $A\gets\lambda I_{kd}$, $b\gets0$, $\mathcal A\gets\varnothing$, $h^\star\gets\bot$
\State Initialize use counts and lineage records
\For{$t=1,\ldots,T$}
    \State Form $U_t$ using Eq.~(\ref{eq:acquisition})
    \State $h\gets\Call{MixedStart}{\mathcal P,\mathcal A,h^\star}$
    \State $h\gets\Call{CoordinateSearch}{h,U_t,\mathcal P,S}$
    \State $h_t\gets\Call{NoveltyFallback}{h,\mathcal P,\mathcal A}$
    \State $(r_t,\mathcal T_t,\kappa_t)\gets\Call{Evaluate}{M,h_t}$
    \State Update $(A,b)$ using Eq.~(\ref{eq:update})
    \State Archive $(h_t,r_t,\mathcal T_t,\kappa_t)$; update usage
    \State $h^\star\gets\Call{SelectVal}{\mathcal A}$
    \If{$t<T$ \textbf{and} $t\bmod L=0$}
        \State $\mathcal Z\gets\Call{Diagnose}{\mathcal T_t,\mathcal A}$
        \State $\mathcal C\gets\Call{ProposeChecked}{\mathcal P,\mathcal I,\mathcal Z}$
        \ForAll{accepted $(j,c'_j)\in\mathcal C$}
            \State $c_j^-\gets\Call{Evict}{P_j,\mathcal A,h^\star}$
            \State $P_j\gets(P_j\cup\{c'_j\})\setminus\{c_j^-\}$
            \State Cache $e(c'_j)$ and lineage; $n(c'_j)\gets0$
        \EndFor
    \EndIf
\EndFor
\State \Return $h^\star$
\end{algorithmic}
\end{wrapalgorithm}

MESH-Harness combines bandit-guided configuration selection
with trace-driven module evolution. Validation scores update
a shared performance model, while execution traces guide
changes to the candidate pools.
Algorithm~\ref{alg:mesh-harness} summarizes the procedure;
Appendix~\ref{app:algorithm} specifies its state transitions and interfaces.

\subsection{Module Pools and Shared Representations}
\label{sec:mesh-representation}

An LLM initializes $N$ implementations per slot from the task
description, interfaces, and permitted training information.
Candidates \textcolor{black}{for a retrieval slot, for example, may rank evidence differently}. They must pass \textcolor{black}{their slot's interface checks} and basic
execution checks; duplicates are rejected, and invalid
implementations are revised using error feedback.

A fixed encoder maps each \textcolor{black}{candidate's} role description, interface,
and implementation summary to an embedding, followed by a fixed
projection and normalization to obtain $e(c)\in\mathbb R^d$.
The configuration representation concatenates embeddings in
slot order:
\begin{equation}
    x(h)=[e(c_1);\ldots;e(c_k)]\in\mathbb R^{kd}.
    \label{eq:representation}
\end{equation}
\textcolor{black}{The shared surrogate (i.e., $\hat\mu_t(h)$ in Section \ref{sec:mesh-search}) predicts harness performance from $x(h)$, whose fixed slot order aligns module implementations by functional role across configurations.} An observation from one configuration can therefore
affect the ranking of other combinations, including those that
have not been executed. Embeddings are cached; newly proposed
\textcolor{black}{implementations} are encoded in the same space without resetting the model.

\subsection{Bandit-Guided Compositional Search}
\label{sec:mesh-search}

Following linear contextual
bandits~\citep{li2010contextual,chu2011linear}, we maintain
shared reward statistics over configuration features.
Let $p=kd$ and $x_s=x(h_s)$. Before round $t$,
\begin{equation}
    \meshEqPair
    {A_{t-1}=\lambda I_p+\sum_{s<t}x_sx_s^\top,}
    {b_{t-1}=\sum_{s<t}r_sx_s,}
    \label{eq:statistics}
\end{equation}
where $\lambda>0$, $A_0=\lambda I_p$, and $b_0=0$.
With $\hat\theta_{t-1}=A_{t-1}^{-1}b_{t-1}$, the predicted
performance and exploration signal are
\begin{equation}
    \meshEqPair
    {\hat\mu_t(h)=x(h)^\top\hat\theta_{t-1},}
    {\sigma_t(h)=\sqrt{x(h)^\top A_{t-1}^{-1}x(h)}.}
    \label{eq:prediction}
\end{equation}
We retain the full information matrix, including cross-slot
feature co-occurrences, for joint parameter estimation and
configuration-level uncertainty. The predicted mean remains
additive over slots:
$\hat\mu_t(h)=\sum_{j=1}^{k}e(c_j)^\top\hat\theta_{j,t-1}$,
where $\hat\theta_{j,t-1}$ is the parameter block for slot $j$.
\textbf{The shared reward model and uncertainty estimate guide the
selection of complete configurations}; actual performance is
determined by end-to-end execution.
Appendices~\ref{app:transfer_identity}--\ref{app:covariance}
derive observation transfer and cross-slot uncertainty.

To account for module coverage and lineage reuse as the
pools evolve, we use
\begin{equation}
    U_t(h)=\hat\mu_t(h)+\alpha g_t(h)\sigma_t(h)
           +B_t(h)-F_t(h),
    \label{eq:acquisition}
\end{equation}
with
\begin{equation}
    \meshEqPair
    {B_t(h)=\frac{\eta}{k}\sum_{j=1}^{k}
        \frac{1}{\sqrt{1+n_{t-1}(c_j)}},}
    {F_t(h)=\frac{\gamma}{k}\sum_{j=1}^{k}
        \log\!\left(1+u_{j,t-1}(f(c_j))\right).}
    \label{eq:coverage-lineage}
\end{equation}
Here, $n_{t-1}(c)$ counts prior evaluations containing $c$,
$f(c)$ identifies its lineage family, and $u_{j,t-1}(f)$
counts that family's uses in slot $j$ over the latest $W$
evaluations. Thus, \emph{$B_t$ favors underused modules and $F_t$
discourages repeated lineage selection}.

For similarity-aware exploration, set
$g_t(h)=\max\{g_{\min},1-\beta\bar s_t(h)\}$.
Across slots, $\bar s_t(h)$ averages each candidate's maximum
nonnegative cosine similarity to previously evaluated modules
in the same slot and its parent when available; missing
references contribute zero. This reduces the exploration
premium for close variants. The coefficients
$\alpha,\eta,\gamma,\beta\geq0$ and $g_{\min}\in(0,1]$
control these terms, which affect selection only, not the
reward update.

We approximately maximize $U_t$ by mixed-start coordinate
ascent, building on the coordinate search used in
MASPOB~\citep{hong2026maspob}. Each round draws one
start from a mixture of the incumbent, the top-$K$ archived
configurations restorable by the task adapter, and a
random combination; an empty archive gives a random start.
The incumbent supports local refinement, while archived and
random starts provide alternative search points. We then
sweep through the slots:
\begin{equation}
    c_j\leftarrow
    \operatorname*{arg\,max}_{c\in P_{j,t}}
    U_t(h_{-j},c),\qquad j=1,\ldots,k,
    \label{eq:coordinate}
\end{equation}
where $(h_{-j},c)$ \textcolor{black}{substitutes candidate $c$ for the implementation
in slot $j$} and retains the latest choices elsewhere. Each update scores a complete configuration.
We use two sweeps by default, favor unevaluated configurations
on ties, and use random fallback sampling to seek a new
configuration when the result has already been evaluated.
The $S$ sweeps require $O(S\sum_j|P_{j,t}|)$ surrogate
evaluations (Appendix~\ref{app:coordinate}). Throughout these sweeps, $U_t$ remains fixed
and no harness is executed. Only the selected $h_t$ receives
a task evaluation, after which
\begin{equation}
    \meshEqPair
    {A_t=A_{t-1}+x_tx_t^\top,}
    {b_t=b_{t-1}+r_tx_t.}
    \label{eq:update}
\end{equation}
We also update the archive, module-use counts, and
lineage statistics.

\subsection{Trace-Driven Module Evolution}
\label{sec:mesh-evolution}

After every $L$ evaluations, while budget remains, MESH-Harness
groups validation traces by failure type and identifies slots
\textcolor{black}{whose implementations warrant revision}. Diagnosis uses current failures, recurring
errors, and previous edits in the relevant lineages. The
proposer mixes failure-guided repair, exploration of
less-covered designs, and local refinement of strong
configurations. Each refresh targets two distinct slots and
proposes one child \textcolor{black}{implementation} per slot from its interface, parent code,
and execution evidence. Proposals undergo the checks in
Section~\ref{sec:mesh-representation}, with a bounded number
of retries for invalid implementations. A child \textcolor{black}{implements its parent's slot and can be combined with
candidates for other slots from other rounds}.

Each accepted child replaces one \textcolor{black}{candidate in its slot's pool},
keeping the pool size fixed:
\begin{equation}
    P_{j,t+1}
    =(P_{j,t}\cup\{c'_j\})\setminus\{c_j^{-}\}.
    \label{eq:pool-refresh}
\end{equation}
Among eligible \textcolor{black}{candidates}, we evict the one with the lowest
retention value: the mean observed reward of configurations
containing it. Unevaluated \textcolor{black}{candidates} have value zero and do
not inherit parent scores. Ties favor removing the \textcolor{black}{candidate}
most similar to others in the pool. The incumbent's \textcolor{black}{selected implementations}
and the newly inserted child are protected during the update;
historical implementations and lineage records remain
archived. Pools without an accepted child are unchanged.

Refreshes preserve $A_t$ and $b_t$. A new \textcolor{black}{implementation} receives
no immediate task evaluation: subsequent bandit decisions
determine \textcolor{black}{when it is selected and which implementations fill
the other slots}. Search
therefore learns both which combinations to evaluate and
which implementations to add, without assigning an
evaluation to every proposal.
The final harness is chosen by the prespecified validation
criterion and its implementations are frozen for testing.
Generation and embedding costs count toward search
expenditure; deployment executes only the selected harness,
without further search.
\noindent
Algorithm~\ref{alg:mesh-harness} omits time subscripts on
mutable state. Appendix~\ref{app:algorithm_refresh} details
pool updates, and Appendix~\ref{app:selection} specifies
the final selection criterion.

\section{Experiments}
\label{exp}

We evaluate final performance under a limited candidate-evaluation budget, examine the contributions of selection and module evolution, and assess the cost and reuse of the resulting harnesses.

\subsection{Experimental Setup}
\label{sec:experimental_setup}

\paragraph{Tasks and evaluation.}
We study text tasks on LawBench~\citep{fei2024lawbench}, USPTO~\citep{schneider2016whats}, and Symptom-to-Diagnosis dataset~\citep{gretelai2023symptom}; retrieval-augmented mathematics constructed from OlympiadBench~\citep{he2024olympiadbench} and MathArena~\citep{dekoninck2026matharena}; code generation on LiveCodeBench~\citep{jain2025livecodebench}; and interaction in ScienceWorld~\citep{wang2022scienceworld}.
We report the macro-average over the three text datasets, overall Math pass@1 on the combined test set, LiveCodeBench solve rate under the fixed cap-8 protocol, and mean ScienceWorld progress. Math pass@1 divides the total number of correct answers across both datasets by their total number of test instances, weighting each instance equally (Appendix~\ref{app:evaluation}).
All methods use the same splits (Appendix~\ref{app:data}). Search and final harness selection use validation data; the selected implementations are frozen before held-out testing. All reported test scores are means over three test runs of the same frozen harness.

\paragraph{Comparisons and budget.}
We compare with SkillOpt~\citep{yang2026skillopt}, DGM-H~\citep{zhang2026hyperagents}, ADIAS~\citep{jiang2026adias}, AHE~\citep{lin2026ahe}, and Meta-Harness~\citep{lee2026metaharness}.
All optimization methods use the same execution model and a budget of 40 complete candidate evaluations per task. Meta-Harness evaluates $20\times2$ candidates, and MESH-Harness evaluates $40\times1$ compositions.
Meta-Harness starts from fixed harnesses, including few-shot~\citep{brown2020language} and BM25~\citep{robertson2009bm25}, in Tables~\ref{tab:text_seeds}--\ref{tab:math_seeds}, or handcrafted ReAct~\citep{yao2023react} harnesses for LiveCodeBench and ScienceWorld; MESH-Harness initializes \textcolor{black}{per-slot} pools with LLM-generated implementations.

\paragraph{Implementation.}
Both optimizers use Claude Code~\citep{anthropic2026claudecode} with Claude Opus 5~\citep{anthropic2026opus5} at \texttt{medium} effort. Unless stated otherwise, task execution uses DeepSeek V4 Flash~\citep{deepseekai2026deepseekv4} with thinking disabled.
MESH-Harness uses Qwen3-Embedding-8B~\citep{zhang2025qwen3embedding}, 10 candidates per slot, $\alpha=0.8$, and two coordinate-search sweeps per evaluation. Every five evaluations, while budget remains, it proposes one child \textcolor{black}{implementation} in each of two distinct slots. Appendix~\ref{app:settings} gives the full settings and execution limits.

\subsection{Performance under a Fixed Evaluation Cap}
\label{sec:main_results}

\par\noindent
\begin{minipage}{\linewidth}
\centering
\captionof{table}{\textbf{Held-out test performance.} All methods use 40 candidate evaluations; scores average three test runs. Text: dataset macro-average; Math: pooled pass@1; LiveCodeBench: cap-8 solve rate; ScienceWorld: mean progress. Scores use a 0--100 scale. Bold marks the column maximum; $\Delta$ compares MESH-Harness with Meta-Harness.}
\label{tab:main_results}
\small
\setlength{\tabcolsep}{5pt}
\renewcommand{\arraystretch}{1.08}
\begin{tabular*}{\linewidth}{@{\extracolsep{\fill}}lrrrr@{}}
\toprule
Method & Text & Math & LiveCodeBench & ScienceWorld \\
\midrule
SkillOpt     & 53.27 & 46.40 & 57.00 & 53.00 \\
DGM-H        & 42.78 & 53.00 & 72.00 & 52.20 \\
ADIAS        & 42.50 & 52.60 & 69.00 & 43.10 \\
AHE          & 47.39 & 50.60 & 69.00 & 49.00 \\
Meta-Harness & 50.63 & 46.20 & 75.00 & 53.10 \\
\midrule
\textbf{MESH-Harness} & \textbf{56.33} & \textbf{53.21} & \textbf{77.00} & \textbf{58.10} \\
$\Delta$ vs. Meta-Harness & $+5.70$ & $+7.01$ & $+2.00$ & $+5.00$ \\
\bottomrule
\end{tabular*}
\end{minipage}
\par

\paragraph{Main comparison.}
MESH-Harness achieves the highest reported score on each of the four aggregate metrics (Table~\ref{tab:main_results} and Figure~\ref{fig:overall_results}).
Relative to Meta-Harness, the gains are 5.70 points on Text, 7.01 on Math, 2.00 on LiveCodeBench, and 5.00 on ScienceWorld.
The advantage therefore spans context construction, retrieval, execution feedback, and sequential interaction without changing the execution model.
MESH-Harness is particularly strong on Text and ScienceWorld, while reaching the performance level of the strongest Math competitor, DGM-H (53.21 versus 53.00).

\paragraph{Where the gains arise.}
MESH-Harness exceeds the strongest fixed references by 6.33 points on Text and 10.83 on Math (Tables~\ref{tab:text_seeds}--\ref{tab:math_seeds}).
The text improvement over Meta-Harness is driven by LawBench ($+8.00$) and USPTO ($+11.00$), offsetting a decrease on Symptom-to-Diagnosis (84.90 to 83.00).
Both math subsets improve: OlympiadBench reaches 63.00 ($+4.30$), and MathArena reaches 43.10 ($+9.80$).
The gains also extend beyond context-based tasks: MESH-Harness exceeds the handcrafted references on LiveCodeBench (77.00 versus 71.00) and ScienceWorld (58.10 versus 55.50).
Appendix~\ref{app:harness_analysis} analyzes the validation-selected harnesses and their task-specific behavior.

\par\noindent
\begin{minipage}{\linewidth}
\begin{minipage}[t]{0.48\textwidth}
\centering
\captionof{table}{\textbf{Text reference harnesses.} Fixed Meta-Harness seeds appear above the divider; scores are macro-averages.}
\label{tab:text_seeds}
\small
\renewcommand{\arraystretch}{1.06}
\begin{tabular*}{\linewidth}{@{\extracolsep{\fill}}lr@{}}
\toprule
Harness & Score \\
\midrule
No memory      & 26.07 \\
Few-shot (8)   & 41.00 \\
Few-shot (32)  & 46.87 \\
Few-shot (all) & 50.00 \\
ACE memory     & 50.00 \\
MCE memory     & 48.93 \\
\midrule
Meta-Harness & 50.63 \\
\textbf{MESH-Harness} & \textbf{56.33} \\
\bottomrule
\end{tabular*}
\end{minipage}\hfill
\begin{minipage}[t]{0.48\textwidth}
\centering
\captionof{table}{\textbf{Math reference harnesses.} Fixed Meta-Harness seeds appear above the divider; scores are overall pass@1.}
\label{tab:math_seeds}
\small
\renewcommand{\arraystretch}{1.06}
\begin{tabular*}{\linewidth}{@{\extracolsep{\fill}}lr@{}}
\toprule
Harness & Score \\
\midrule
No retrieval    & 23.20 \\
Random few-shot & 41.42 \\
BM25            & 42.18 \\
ACE             & 42.38 \\
\midrule
Meta-Harness & 46.20 \\
\textbf{MESH-Harness} & \textbf{53.21} \\
\bottomrule
\end{tabular*}
\end{minipage}
\end{minipage}
\par

\begin{figure}[!t]
\centering
\includegraphics[width=0.96\textwidth]{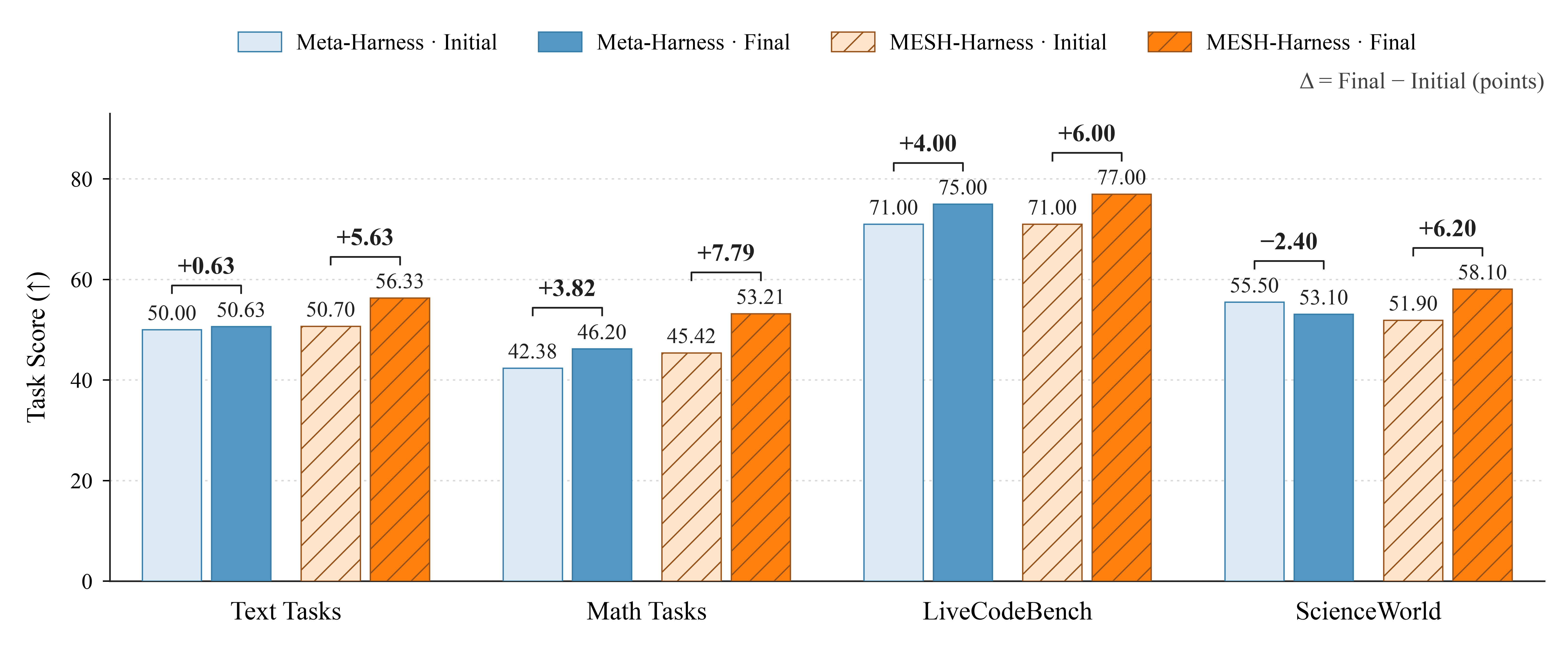}
\captionof{figure}{\textbf{Final performance relative to method-specific references.} ``Initial'' denotes a fixed reference seed for Meta-Harness and the first evaluated composition for MESH-Harness, not a shared initialization. Meta-Harness references are 50.00 (Text), 42.38 (Math, ACE), 71.00 (LiveCodeBench), and 55.50 (ScienceWorld). Brackets show final minus reference.}
\label{fig:search_gains}
\end{figure}
\par

\paragraph{Improvement beyond initialization.}
Figure~\ref{fig:search_gains} shows that MESH-Harness improves by 5.63--7.79 points over its first evaluated compositions.
On ScienceWorld, it rises from 51.90 to 58.10, despite beginning below the handcrafted reference at 55.50.
These results establish that the final performance is not simply inherited from the first composition.
The reference scores describe within-method changes; the matched-cap comparison is the final performance in Table~\ref{tab:main_results}.

\subsection{Contributions of Selection and Evolution}
\label{sec:ablation_study}

The following ablations use the three text tasks, the default executor, and the candidate-evaluation cap from Section~\ref{sec:experimental_setup}.

\paragraph{Selecting which compositions to evaluate matters.}
\textcolor{black}{Randomly selecting one implementation per slot} scores 48.80, 7.53 points below the full method (Table~\ref{tab:bandit_ablation}).
Removing only the UCB uncertainty term ($\alpha=0$), while retaining predicted reward, coverage, and lineage terms, yields 51.70.
Thus, guided selection without the uncertainty term improves on random selection, but still trails the full acquisition rule by 4.63 points.
Performance increases from $\alpha=0$ to $0.8$ and then falls at $1.6$ (Figure~\ref{fig:exploration_ablation}), favoring a moderate exploration bonus among the tested values.

\par\noindent
\begin{minipage}{\linewidth}
\begin{minipage}[t]{0.51\textwidth}
\centering
\captionof{figure}{\textbf{Exploration strength.} Text macro-average versus $\alpha$; the star marks the default. Reference lines show Meta-Harness and random selection.}
\label{fig:exploration_ablation}
\includegraphics[width=\linewidth,trim=17.5bp 12.5bp 14bp 10bp,clip]{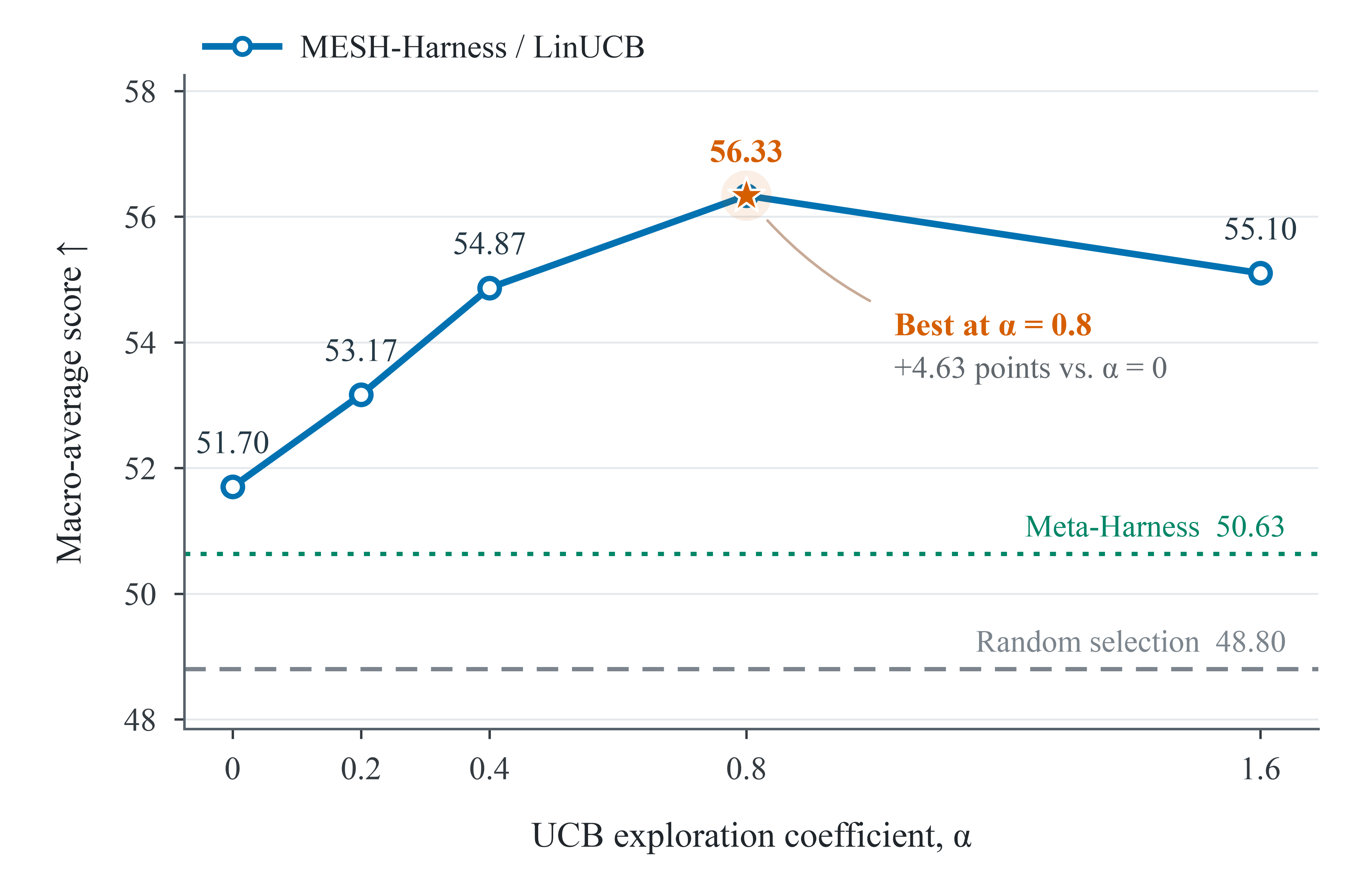}
\end{minipage}\hfill
\begin{minipage}[t]{0.47\textwidth}
\centering
\captionof{table}{\textbf{Configuration selection.} Text macro-average; $\Delta$ is relative to full MESH-Harness. Rows follow the $\alpha$ sweep at left.}
\label{tab:bandit_ablation}
\small
\setlength{\tabcolsep}{2pt}
\renewcommand{\arraystretch}{1.12}
\begin{tabular*}{\linewidth}{@{\extracolsep{\fill}}lrr@{}}
\toprule
Selection & Score & $\Delta$ \\
\midrule
No UCB bonus ($\alpha=0$) & 51.70 & $-4.63$ \\
$\alpha=0.2$ & 53.17 & $-3.16$ \\
$\alpha=0.4$ & 54.87 & $-1.46$ \\
\textbf{Full ($\alpha=0.8$)} & \textbf{56.33} & 0.00 \\
$\alpha=1.6$ & 55.10 & $-1.23$ \\
\midrule
Random composition & 48.80 & $-7.53$ \\
Meta-Harness & 50.63 & $-5.70$ \\
\bottomrule
\end{tabular*}
\end{minipage}
\end{minipage}
\par

\begin{wraptable}{r}{0.56\textwidth}
\centering
\caption{\textbf{Module evolution.} Text macro-average over three test runs and change from the full method.}
\label{tab:evolution_ablation}
\small
\setlength{\tabcolsep}{3pt}
\renewcommand{\arraystretch}{1.06}
\begin{tabular*}{\linewidth}{@{\extracolsep{\fill}}lrr@{}}
\toprule
Variant & Score & $\Delta$ \\
\midrule
\textbf{MESH-Harness} & \textbf{56.33} & 0.00 \\
No pool refresh & 51.20 & $-5.13$ \\
\shortstack[l]{Refresh without\\historical traces} & 53.80 & $-2.53$ \\
\midrule
Meta-Harness & 50.63 & $-5.70$ \\
\bottomrule
\end{tabular*}
\end{wraptable}

\paragraph{Evolving the pools improves performance.}
Keeping the initial pools \textcolor{black}{for all slots} fixed yields 51.20, a 5.13-point drop (Table~\ref{tab:evolution_ablation}).
Refreshing them without historical traces, while retaining pool refresh and bandit statistics, recovers 2.60 points; restoring trace access adds a further 2.53.
Together with the selection ablation, these results support the two parts of the design: use reward feedback to allocate evaluations among compositions, and execution evidence to improve the implementations available for future composition.

\subsection{Cost and Reuse of Optimized Harnesses}
\label{sec:cost_reuse}

\paragraph{Search and execution costs.}
Table~\ref{tab:search_execution_cost} separates optimization-inclusive costs from execution costs.
Across the four reported workloads, MESH-Harness reduces search-plus-test expenditure from \$195.0537 to \$166.6078 (14.6\%) and test-only expenditure from \$3.5528 to \$1.9835 (44.2\%).
Both costs are lower on every task; the largest test-time savings occur on LiveCodeBench (27.0\%) and ScienceWorld (59.5\%).
The final score improvements therefore do not rely on more expensive execution of the selected harnesses.

\par\noindent
\begin{minipage}{\linewidth}
\centering
\captionof{table}{\textbf{Search and execution costs (USD).} Search + test includes optimization and testing the selected configurations; test only excludes optimization. Savings are relative to Meta-Harness. Totals sum the four reported task workloads.}
\label{tab:search_execution_cost}
\small
\setlength{\tabcolsep}{3pt}
\renewcommand{\arraystretch}{1.0}
\begin{tabular*}{\linewidth}{@{\extracolsep{\fill}}lrrrrrr@{}}
\toprule
& \multicolumn{3}{c}{Search + test} & \multicolumn{3}{c}{Test only} \\
\cmidrule(lr){2-4}\cmidrule(l){5-7}
Task & Meta & MESH & Saving & Meta & MESH & Saving \\
\midrule
Text & 63.1843 & 51.1643 & 19.0\% & 0.1709 & 0.1662 & 2.8\% \\
Math & 58.8652 & 47.2431 & 19.7\% & 0.5475 & 0.4987 & 8.9\% \\
LiveCodeBench & 37.1387 & 34.7826 & 6.3\% & 0.5273 & 0.3847 & 27.0\% \\
ScienceWorld & 35.8655 & 33.4178 & 6.8\% & 2.3071 & 0.9339 & 59.5\% \\
\midrule
\textbf{Total} & \textbf{195.0537} & \textbf{166.6078} & \textbf{14.6\%} & \textbf{3.5528} & \textbf{1.9835} & \textbf{44.2\%} \\
\bottomrule
\end{tabular*}
\end{minipage}
\par

\par\noindent
\begin{minipage}{\linewidth}
\begin{minipage}[t]{0.49\textwidth}
\centering
\captionof{figure}{\textbf{Score--cost tradeoff.} Circles: MESH-Harness; squares: Meta-Harness. The dashed line connects nondominated evaluated configurations. Cost uses a log scale.}
\label{fig:execution_pareto}
\includegraphics[width=\linewidth,trim=22.5bp 18bp 11.5bp 11bp,clip]{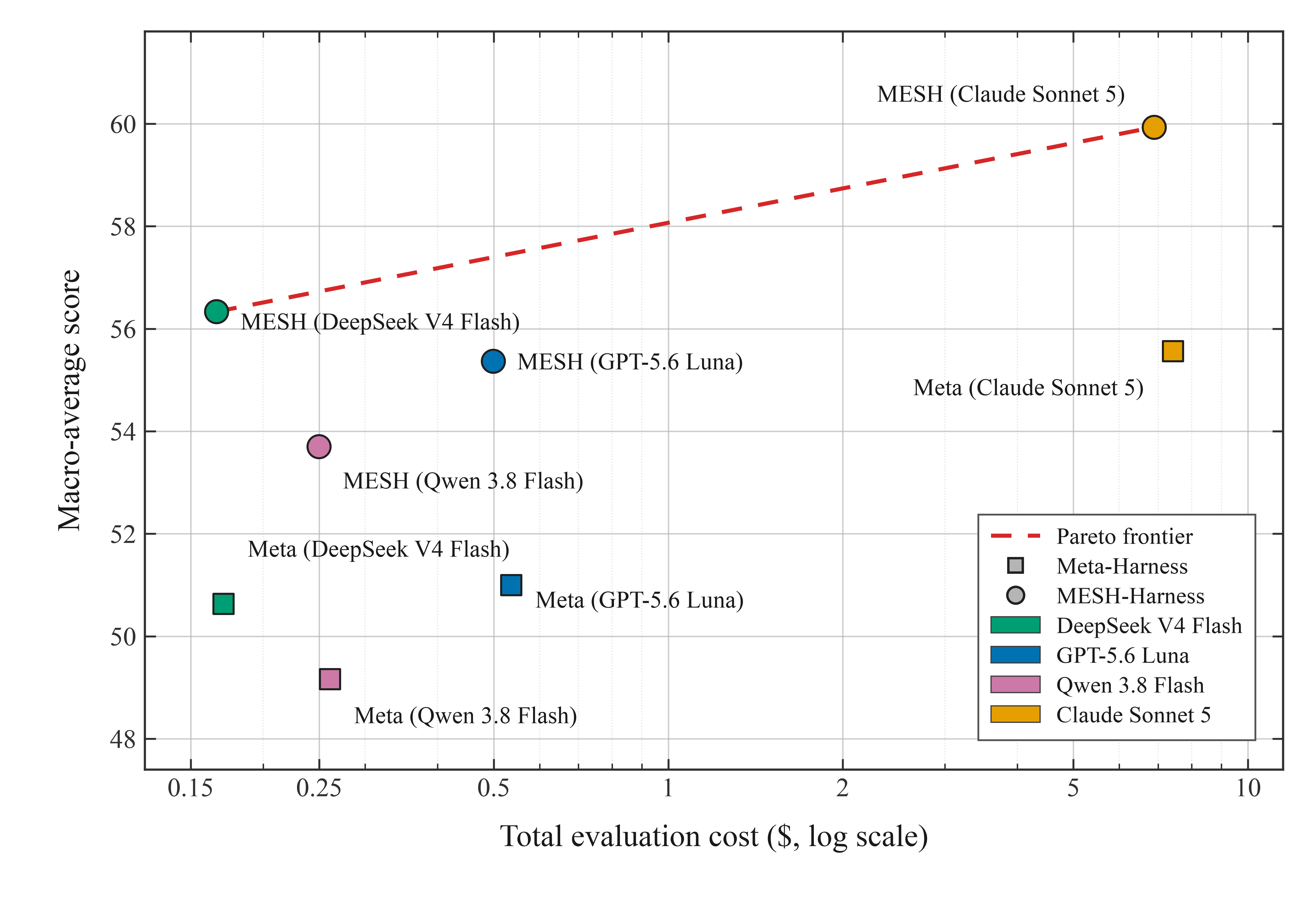}
\end{minipage}\hfill
\begin{minipage}[t]{0.49\textwidth}
\centering
\captionof{table}{\textbf{Frozen-harness transfer.} Text macro-average and execution cost (USD), without further search. Bold marks the better result within each model.}
\label{tab:execution_models}
\small
\setlength{\tabcolsep}{1.5pt}
\renewcommand{\arraystretch}{1.20}
\begin{tabular*}{\linewidth}{@{\extracolsep{\fill}}lrrrr@{}}
\toprule
& \multicolumn{2}{c}{Score $\uparrow$} & \multicolumn{2}{c}{Cost $\downarrow$} \\
\cmidrule(lr){2-3}\cmidrule(l){4-5}
Execution model & Meta & MESH & Meta & MESH \\
\midrule
\shortstack[l]{DeepSeek V4\\Flash} & 50.63 & \textbf{56.33} & 0.1709 & \textbf{0.1662} \\
\addlinespace[3pt]
\shortstack[l]{GPT-5.6\\Luna} & 51.00 & \textbf{55.37} & 0.5359 & \textbf{0.4987} \\
\addlinespace[3pt]
\shortstack[l]{Qwen 3.8\\Flash} & 49.17 & \textbf{53.70} & 0.2608 & \textbf{0.2496} \\
\addlinespace[3pt]
\shortstack[l]{Claude\\Sonnet 5} & 55.57 & \textbf{59.93} & 7.4215 & \textbf{6.8874} \\
\bottomrule
\end{tabular*}
\end{minipage}
\end{minipage}
\par

\paragraph{Reuse across execution models.}
We freeze the text harnesses and replace only the executor with GPT-5.6 Luna~\citep{openai2026gpt56}, Qwen 3.8 Flash~\citep{qwen2026qwen38flash}, or Claude Sonnet 5~\citep{anthropic2026sonnet5}.
Including the default executor, MESH-Harness retains a 4.36--5.70-point advantage and a 2.8--7.2\% cost reduction across all four models (Table~\ref{tab:execution_models}).
Among the measured configurations, its DeepSeek and Sonnet variants are the two nondominated points (Figure~\ref{fig:execution_pareto}), offering the lowest cost and highest score, respectively.

\begin{wraptable}{r}{0.55\textwidth}
\centering
\caption{\textbf{Refining a Meta-Harness solution.} Held-out text scores before and after 20 additional MESH-Harness candidate evaluations.}
\label{tab:meta_initialized_search}
\small
\setlength{\tabcolsep}{3pt}
\renewcommand{\arraystretch}{1.07}
\begin{tabular*}{\linewidth}{@{\extracolsep{\fill}}lrr@{}}
\toprule
Dataset & \shortstack[r]{Meta-Harness\\solution} & \shortstack[r]{+ MESH\\refinement} \\
\midrule
LawBench & 51.00 & \textbf{55.00} \\
Symptom-to-Diagnosis & 84.90 & \textbf{90.10} \\
USPTO & 16.00 & \textbf{18.00} \\
\midrule
Macro-average & 50.63 & \textbf{54.37} \\
\bottomrule
\end{tabular*}
\end{wraptable}

\paragraph{Reuse of an existing optimized harness.}
Starting from a Meta-Harness solution, 20 additional MESH-Harness evaluations improve the text macro-average from 50.63 to 54.37, with gains on all three datasets (Table~\ref{tab:meta_initialized_search}).
Compositional refinement can therefore build on an existing optimizer's output, not only on MESH-Harness's own initialization.
This experiment studies further refinement, separately from the matched-cap comparison.

\WFclear
\FloatBarrier

\section{Conclusion}
\label{sec:conclusion}

\looseness=-1
We introduced MESH-Harness, which combines shared module representations, bandit-guided compositional search, and trace-driven module evolution to improve agent harnesses without changing model weights.
Across four task domains, it achieves the highest aggregate scores among the evaluated methods and outperforms Meta-Harness under matched candidate-evaluation budgets.
Text ablations support both guided selection and module evolution, while further experiments on these tasks show transfer across execution models and gains from refining a Meta-Harness solution.
For the reported configurations, aggregate test-time cost is 44.2\% lower than that of Meta-Harness, and combined search-and-test expenditure is 14.6\% lower.
These results show that reusing and refining \textcolor{black}{implementations within fixed functional slots} can improve task performance while reducing the execution cost of the resulting harnesses.

\par
\endgroup
\setcounter{topnumber}{2}

\clearpage
\section*{AI Use Statement}

We used generative AI tools to improve the clarity, grammar, and presentation of the manuscript and to assist with implementing, debugging, and refining research code. Generative models also serve as the proposer and executor in MESH-Harness, as described in the method and experimental setup. The authors reviewed and revised AI-assisted text and inspected and tested AI-assisted code. The authors take full responsibility for the final manuscript, implementation, and experimental results.

\section*{Ethics Statement}

This work studies the optimization of agent harnesses with existing language models and benchmark datasets spanning text tasks, mathematical reasoning, code generation, and simulated scientific interaction. The experiments evaluate model outputs and executable programs within the specified benchmark protocols; they do not involve recruiting human participants or deploying systems in real-world decision-making. Diagnosis, legal, and chemical outputs are used only as benchmark predictions. We are not aware of specific ethical concerns arising from the experiments reported in this study.

\section*{Reproducibility Statement}

The main text and appendices detail the MESH-Harness design and experiments.\newline
Algorithm~\ref{alg:mesh-harness} and Section~\ref{sec:method} specify the MESH-Harness procedure, while Appendix~\ref{app:theory} provides the theoretical analysis and Appendix~\ref{app:algorithm} expands the optimization loop, composition selection, and module refresh procedures. Section~\ref{sec:experimental_setup} identifies the models, benchmarks, baselines, and evaluation protocol. Appendix~\ref{app:implementation} details the data splits, retrieval resources, memory boundaries, search hyperparameters, execution limits, scoring rules, and cost accounting. Appendix~\ref{app:harness_analysis} describes the selected harnesses and their task-specific control flow. All comparisons use the shared task splits and scoring protocols described in the paper. Final harnesses are selected on validation data and frozen before testing; reported test scores average three executions of each frozen harness. The models, benchmarks, and retrieval sources are identified and cited in the manuscript.

\bibliography{main}
\bibliographystyle{iclr2027_conference}

\clearpage
\appendix

\raggedbottom
\renewcommand{\eqref}[1]{\textup{(\ref{#1})}}
\section{Theoretical Supplements}
\label{app:theory}

This section analyzes the mechanisms that make compositional search useful under a limited evaluation budget: information transfer through a shared surrogate, uncertainty over module combinations, and inexpensive coordinate-wise selection. The first two results follow from the ridge state and the search rule themselves. The final result explicitly states the additional confidence-envelope assumption needed to relate acquisition-based selection to the task objective.

\subsection{How one evaluation informs other compositions}
\label{app:transfer_identity}

Let $p=kd$, and fix the embedding map throughout a search run. At a given round, write
\begin{equation}
 A=\lambda I+\sum_{s<t}x_sx_s^\top,\qquad
 b=\sum_{s<t}r_sx_s,\qquad \widehat\theta=A^{-1}b,
 \label{eq:app-ridge-state}
\end{equation}
where $\lambda>0$. Since each slot embedding has unit norm, $\|x(h)\|_2^2=k$. In particular, $A$ is positive definite and $x(h)^\top A^{-1}x(h)\leq k/\lambda$ for every composition, including a composition that has not been evaluated.

\begin{meshproposition}[Transfer through the shared ridge state]
\label{prop:app-transfer}
Suppose the next evaluated harness has feature vector $x$ and reward $r$. Let $A^+=A+xx^\top$ and $b^+=b+rx$. For any fixed comparison vector $v$, the posterior mean and squared uncertainty satisfy
\begin{align}
 v^\top\widehat\theta^+-v^\top\widehat\theta
 &=\frac{v^\top A^{-1}x}{1+x^\top A^{-1}x}
   \bigl(r-x^\top\widehat\theta\bigr),
 \label{eq:app-mean-transfer}\\
 v^\top(A^+)^{-1}v
 &=v^\top A^{-1}v-
   \frac{(v^\top A^{-1}x)^2}{1+x^\top A^{-1}x}.
 \label{eq:app-variance-transfer}
\end{align}
\end{meshproposition}

\begin{proof}
The Sherman--Morrison identity~\citep{sherman1950adjustment} gives
\begin{equation}
 (A+xx^\top)^{-1}
 =A^{-1}-\frac{A^{-1}xx^\top A^{-1}}{1+x^\top A^{-1}x}.
 \label{eq:app-sm}
\end{equation}
Multiplying by $b+rx$ and using $A^{-1}b=\widehat\theta$ yields
$\widehat\theta^+=\widehat\theta+
 A^{-1}x(r-x^\top\widehat\theta)/(1+x^\top A^{-1}x)$.
Left multiplication by $v^\top$ proves Eq.~\eqref{eq:app-mean-transfer}; taking the quadratic form of Eq.~\eqref{eq:app-sm} proves Eq.~\eqref{eq:app-variance-transfer}.
\end{proof}

These identities make the reuse mechanism explicit. An evaluation changes another composition's prediction through its overlap with the evaluated direction in the $A^{-1}$ geometry, not merely through exact module identity. The update also depends on the prediction residual: a surprising result has a different effect from one already explained by the surrogate. Uncertainty cannot increase at a fixed vector after an observation, and its reduction is largest along directions coupled to the evaluated composition. This is surrogate-level information transfer; it is not an assignment of the observed reward to every selected module as an independent contribution.

The same identities apply after a pool refresh. A new implementation is encoded in the existing feature space, so it can be scored immediately with the accumulated $(A,b)$. Its creation does not add an observation, and it receives neither its parent's reward nor a synthetic reward. Retiring an implementation changes the feasible set of compositions but does not invalidate the stored observations. This separation between the feasible set and the statistical state is what allows the pools to evolve without restarting the surrogate.

\subsection{The role of full covariance in compositional uncertainty}
\label{app:covariance}

Partition $\widehat\theta$ and $A$ by functional slot. The mean prediction remains additive:
\begin{equation}
 \widehat\mu(h)=\sum_{j=1}^{k}\widehat\theta_j^\top e(c_j).
 \label{eq:app-additive-mean}
\end{equation}
Full covariance does not introduce explicit pairwise reward features. Rather, its off-diagonal blocks retain which slot representations have been observed together:
\begin{equation}
 A_{ij}=\lambda\,\mathbf{1}\{i=j\}I_d+
 \sum_{s<t}e(c_{i,s})e(c_{j,s})^\top.
 \label{eq:app-cross-block}
\end{equation}
These blocks affect both the fitted coefficients and the uncertainty of a proposed combination.

A two-slot example isolates the distinction. Let $d=1$, and suppose all $n$ observations used $x_+=(1,1)^\top$. Then
\begin{equation}
 A=\lambda I+n x_+x_+^\top,\qquad
 \sigma^2(x_+)=\frac{2}{\lambda+2n},\qquad
 \sigma^2(x_-)=\frac{2}{\lambda},\quad x_-=(1,-1)^\top.
 \label{eq:app-cooccurrence-example}
\end{equation}
Repeatedly observing the same pairing resolves the common direction but provides no information about the orthogonal contrast. Replacing $A$ by its block diagonal would assign both directions uncertainty $2/(\lambda+n)$ and lose that distinction. In a compositional search, this matters when modules have repeatedly appeared with the same partners: frequent exposure of each module need not imply that a different pairing is well understood.

The distinction is useful for evaluation allocation: exposure to individual modules and exposure to their joint directions are different forms of information. The embedding-based mean is a surrogate for program performance, while end-to-end execution supplies the observed reward. The selection ablations in Table~\ref{tab:bandit_ablation} complement this algebraic analysis by measuring the effect of the acquisition rule under the fixed evaluation cap.

\subsection{Monotonicity and computational cost of coordinate search}
\label{app:coordinate}

During one selection step, the pools, posterior, usage counts, and lineage statistics are held fixed. Thus the acquisition $U_t(h)$ is a fixed function on the current Cartesian product. A coordinate update enumerates the active candidates for one slot while retaining the other slots.

\begin{meshproposition}[Monotonic coordinate updates]
\label{prop:app-coordinate}
If each coordinate update includes the currently selected module among its choices and selects an acquisition maximizer, the acquisition is nondecreasing over the coordinate updates. If the procedure is continued using only strict improving moves, it terminates at a coordinate-wise maximizer on a finite product space.
\end{meshproposition}

\begin{proof}
The current composition is feasible in every coordinate subproblem, so the selected replacement has acquisition at least as large. With strict improving moves, a composition cannot be revisited. Finiteness therefore implies termination; at termination no single-coordinate replacement improves the acquisition.
\end{proof}

The implementation uses two sweeps as a bounded local-ascent procedure. The proposition characterizes each feasible coordinate update; convergence to a coordinate-wise maximizer would additionally require continuing strict improvement to termination. Novelty-based tie breaking preserves nondecrease for equal acquisition values. Duplicate avoidance is a subsequent selection step, detailed in Section~\ref{app:algorithm_selection}.

For slot pool sizes $N_1,\ldots,N_k$ and $S$ sweeps, the coordinate comparisons require $S\sum_jN_j$ acquisition evaluations, plus final scoring and any start-selection or duplicate-handling overhead. Exhaustive comparison would require $\prod_jN_j$ evaluations. With $N_j=10$ and $S=2$, this is 140 coordinate comparisons instead of $10^7$ for the seven-slot tasks, and 100 instead of $10^5$ for the five-slot ScienceWorld controller. These counts concern inexpensive surrogate comparisons, not benchmark executions. Only the selected complete harness consumes a candidate-evaluation slot.

The dense statistical state occupies $O(p^2)$ space. With a cached inverse, a rank-one update and a quadratic-form score each require $O(p^2)$ arithmetic; the agent implementation uses this update. The text and math implementation instead stores $(A,b)$ and evaluates the corresponding dense linear solves. These implementations realize the same statistical update with different linear-algebra costs. The comparison above concerns acquisition-query complexity; embedding newly proposed modules is accounted for as initialization or refresh work.

\subsection{Decomposing uncertainty and approximate selection}
\label{app:acquisition_scope}

Coverage, lineage, and similarity adjustments modify evaluation priorities while leaving the observed rewards unchanged. The following decomposition separates these adjustments from uncertainty and approximate maximization using the confidence-envelope formulation of linear bandit analysis~\citep{abbasiyadkori2011improved}. Consider a round on which, for some $q_t\geq0$, the surrogate satisfies the uniform envelope
\begin{equation}
 |J(h)-\widehat\mu_t(h)|\leq q_t\sigma_t(h)
 \quad\text{for all }h\in\mathcal H_t.
 \label{eq:app-envelope}
\end{equation}
Write $V_t=\widehat\mu_t+q_t\sigma_t$ and $D_t=U_t-V_t$. Let $h_t^*$ maximize $J$ over the current pools, and define the acquisition-search error
$\varepsilon_t=\max_{h\in\mathcal H_t}U_t(h)-U_t(h_t)\geq0$.

\begin{meshproposition}[A conditional one-step decomposition]
\label{prop:app-decomposition}
Under Eq.~\eqref{eq:app-envelope},
\begin{equation}
 J(h_t^*)-J(h_t)
 \leq 2q_t\sigma_t(h_t)
 +\operatorname{osc}_{\mathcal H_t}(D_t)+\varepsilon_t,
 \label{eq:app-decomposition}
\end{equation}
where $\operatorname{osc}(D)=\max D-\min D$.
\end{meshproposition}

\begin{proof}
The envelope gives $J(h_t^*)\leq V_t(h_t^*)$ and
$J(h_t)\geq V_t(h_t)-2q_t\sigma_t(h_t)$. Moreover,
$V_t(h_t^*)-V_t(h_t)
=U_t(h_t^*)-U_t(h_t)+D_t(h_t)-D_t(h_t^*)
\leq\varepsilon_t+\operatorname{osc}(D_t)$.
Combining the inequalities proves the claim.
\end{proof}

For MESH-Harness,
$D_t=(\alpha g_t-q_t)\sigma_t+B_t-F_t$.
Equation~\eqref{eq:app-decomposition} separates three sources of selection error: uncertainty in the chosen composition, the range of the acquisition adjustment, and the gap left by approximate search. It is a conditional statement under Eq.~\eqref{eq:app-envelope}; the empirical $\alpha$ is a search coefficient rather than an asserted calibrated confidence radius. Coverage and lineage regularize exploration of an evolving candidate space, while coordinate search bounds the cost of comparing those candidates.

Pool refresh changes the feasible set without resetting the statistical state. Protecting the incumbent preserves an already evaluated configuration, while retaining historical embeddings allows subsequent selection to reuse all completed observations. Section~\ref{app:algorithm} makes these state transitions explicit, including the different roles of the composition reward, the pool-retention statistic, and the validation criterion used for final selection.

\clearpage
\section{Algorithmic Details}
\label{app:algorithm}

This section expands Algorithm~\ref{alg:mesh-harness} into the procedures used for initialization, composition selection, and pool refresh. The central design is to separate proposing an implementation from deciding to evaluate it. A checked module enters a reusable pool; only a complete composition selected by the bandit receives an end-to-end validation run. The resulting score updates the shared surrogate, while the associated traces inform subsequent code changes.

\subsection{Search state and complete procedure}
\label{app:algorithm_state}

The state consists of the active pools $\mathcal P=\{P_j\}_{j=1}^k$, a catalog $\mathcal V$ retaining all accepted implementations and their cached embeddings, the ridge statistics $(A,b)$, the validation archive $\mathcal A$, and the set $\mathcal K$ of evaluated composition keys. A key is the ordered tuple of component identifiers. Usage records $\mathcal U$ store each module's evaluation count, accumulated reward, parent, lineage family, and selected rounds. An initial module starts its own family; a child retains its parent's family. The encoder and task interfaces remain fixed throughout the run.

Algorithm~\ref{alg:app-complete} uses $t$ to count completed candidate evaluations, starting from one. The task adapter \textsc{EvaluateVal} implements the dataset-specific execution and scoring rules in Section~\ref{app:implementation}. \textsc{SelectVal} applies the prespecified validation criterion in Section~\ref{app:selection}; neither the acquisition score nor a test score is substituted for that criterion.

\begin{algorithm}[H]
\caption{MESH-Harness: complete optimization procedure}
\label{alg:app-complete}
\small
\begin{algorithmic}[1]
\Require Fixed executor $M$, proposer $G$, encoder $E$, ordered interfaces $\mathcal I$
\Require Permitted initialization data $D_{\mathrm{init}}$, training data $D_{\mathrm{tr}}$, validation data $D_{\mathrm{val}}$
\Require Fixed task resources $\mathcal C$; pool size $N$, evaluation budget $T$, refresh interval $L$
\State $\mathcal P,\mathcal V,\mathcal A,\mathcal K,\mathcal U\gets\varnothing$; $h^\star\gets\bot$
\For{$j=1,\ldots,k$}
    \State $P_j\gets\varnothing$
    \For{$i=1,\ldots,N$}
        \State $c\gets\Call{CheckedProposal}{G,I_j,D_{\mathrm{init}},\mathcal V_j,\mathrm{initial}}$ \Comment{Algorithm~\ref{alg:app-checks}}
        \State Cache $e(c)\gets\Call{EncodeNormalize}{E,I_j,c}$; initialize its family and zero usage
        \State $P_j\gets P_j\cup\{c\}$; $\mathcal V_j\gets\mathcal V_j\cup\{c\}$
    \EndFor
\EndFor
\State $A\gets\lambda I_{kd}$; $b\gets0$; initialize the seeded random generator
\For{$t=1,\ldots,T$}
    \State Hold $(A,b,\mathcal P,\mathcal U)$ fixed during the following selection step
    \State $h_t\gets\Call{SelectComposition}{\mathcal P,\mathcal V,A,b,\mathcal A,\mathcal U,\mathcal K,h^\star}$ \Comment{Algorithm~\ref{alg:app-selection}}
    \State $H_t\gets\operatorname{Compose}_{\mathcal I}(h_t)$; $x_t\gets[e(c_{1,t});\ldots;e(c_{k,t})]$
    \State $(r_t,\mathcal T_t,\kappa_t)\gets\Call{EvaluateVal}{M,H_t,D_{\mathrm{tr}},D_{\mathrm{val}},\mathcal C}$
    \State $A\gets A+x_tx_t^\top$; $b\gets b+r_tx_t$
    \State $\mathcal A\gets\mathcal A\cup\{(t,h_t,r_t,\mathcal T_t,\kappa_t)\}$; $\mathcal K\gets\mathcal K\cup\{\operatorname{key}(h_t)\}$
    \ForAll{$c\in h_t$}
        \State $n(c)\gets n(c)+1$; $s(c)\gets s(c)+r_t$; append $t$ to its usage history
    \EndFor
    \State Update recent family uses and committed validation-trace history in $\mathcal U$
    \State $h^\star\gets\Call{SelectVal}{\mathcal A}$
    \If{$t<T$ \textbf{and} $t\bmod L=0$}
        \State $(\mathcal P,\mathcal V,\mathcal U)\gets\Call{RefreshPools}{t,h_t,h^\star,\mathcal P,\mathcal V,\mathcal A,\mathcal U}$
        \Statex \hspace{\algorithmicindent}\hspace{\algorithmicindent}\Comment{Algorithm~\ref{alg:app-refresh}; no task evaluation and no change to $(A,b,\mathcal K)$}
    \EndIf
    \State Persist the completed-round state and resource-use records
\EndFor
\State $h^\star\gets\Call{SelectVal}{\mathcal A}$
\State \Return $\Call{Freeze}{h^\star,\mathcal V,\mathcal I,M,\mathcal C}$ and the validation archive
\end{algorithmic}
\end{algorithm}

\clearpage
\subsection{Mixed-start composition selection}
\label{app:algorithm_selection}

The acquisition in Eq.~\eqref{eq:acquisition} is evaluated on complete compositions, even when only one slot is being changed. For a candidate $c$ in slot $j$, let $\mathcal R_j(c)$ contain previously evaluated modules in that slot and its direct parent, when available. With normalized embeddings, the similarity term is
\begin{equation}
 s_j(c)=\max\bigl(\{0\}\cup\{e(c)^\top e(v):v\in\mathcal R_j(c)\}\bigr),
 \qquad \bar s(h)=\frac1k\sum_{j=1}^k s_j(c_j).
 \label{eq:app-similarity-detail}
\end{equation}
Thus $U(h)$ uses the ridge mean, full-covariance uncertainty scaled by $\max(g_{\min},1-\beta\bar s(h))$, the module-coverage bonus, and the recent-family penalty. All these quantities are computed from cached representations and completed evaluations. Coordinate comparisons do not call the proposer, executor, or task evaluator.

\begin{algorithm}[H]
\caption{Mixed-start coordinate search with novelty-aware tie breaking}
\label{alg:app-selection}
\small
\begin{algorithmic}[1]
\Require Search state from Algorithm~\ref{alg:app-complete}; $S=2$, $K=5$, random-attempt limit $R=200$
\Function{SelectComposition}{$\mathcal P,\mathcal V,A,b,\mathcal A,\mathcal U,\mathcal K,h^\star$}
    \State Construct $U$ from Eqs.~\eqref{eq:acquisition}--\eqref{eq:coverage-lineage} and \eqref{eq:app-similarity-detail}
    \State $\mathcal B\gets$ top-$K$ validation-reward entries restorable by the task's archive adapter
    \State Draw $u\sim\operatorname{Uniform}[0,1)$
    \If{$u<0.65$ \textbf{and} $h^\star\neq\bot$}
        \State $h\gets h^\star$
    \ElsIf{$u<0.85$ \textbf{and} $\mathcal B\neq\varnothing$}
        \State $h\gets$ a uniformly sampled composition from $\mathcal B$
    \Else
        \State $h\gets\Call{RandomNew}{\mathcal P,\mathcal K,R}$
    \EndIf
    \For{$s=1,\ldots,S$}
        \For{$j=1,\ldots,k$}
            \State For every $c\in P_j$, form $h^{(c)}\gets(h_{-j},c)$ and compute $U(h^{(c)})$
            \State $c^+\gets\operatorname*{arg\,max}_{c\in P_j}\bigl(U(h^{(c)}),\mathbf1\{\operatorname{key}(h^{(c)})\notin\mathcal K\},\operatorname{birth}(c),\operatorname{id}(c)\bigr)$
            \State $h\gets(h_{-j},c^+)$ \Comment{Lexicographic maximization; retain latest choices elsewhere}
        \EndFor
    \EndFor
    \If{$\operatorname{key}(h)\in\mathcal K$}
        \State $h\gets\Call{RandomNew}{\mathcal P,\mathcal K,R}$
    \EndIf
    \State \Return $h$
\EndFunction
\Statex
\Function{RandomNew}{$\mathcal P,\mathcal K,R$}
    \For{$a=1,\ldots,R$}
        \State Draw each $c_j$ independently and uniformly from $P_j$; $h\gets(c_1,\ldots,c_k)$
        \If{$\operatorname{key}(h)\notin\mathcal K$}
            \State \Return $h$
        \EndIf
    \EndFor
    \State \Return the last sampled $h$
\EndFunction
\end{algorithmic}
\end{algorithm}

The start mixture is sampled once per evaluation, rather than running three searches and evaluating their winners. The archived-start adapter reconstructs component identities from stored code. The agent runner restricts this bank to fully active compositions; the Text/Math runner can restore historical starts, whose slots are replaced from the active pools during the first sweep. The evaluated composition is in the active product space in either case. The monotonicity result in Proposition~\ref{prop:app-coordinate} applies whenever the current coordinate is among the feasible replacements. Novelty is a tie breaker, not an additional reward, and bounded random sampling is the duplicate fallback rather than an exhaustive search for an unseen composition.

\clearpage
\subsection{Trace-guided proposals and fixed-capacity pool updates}
\label{app:algorithm_refresh}

Refresh ranks validation failure groups by recurring count and then diagnostic confidence. Failure-guided actions scan primary and secondary suspected slots, using the corresponding module in the latest composition. Coverage-oriented actions rank active parents by recent family use and evaluation count. Local-refinement actions rank incumbent modules by recent family use. Remaining ties use the runner's deterministic creation-time and identifier ordering. These alternatives supply different parent designs through the same slot interface.

The default plan first reserves an available failure-guided action, then fills the remaining position from the rotating order $(\mathrm{failure},\mathrm{failure},\mathrm{coverage},\mathrm{local})$. A slot can appear at most once. At refresh $q$, the order is rotated by $(q-1)\bmod4$; refreshes follow evaluations 5, 10, 15, 20, 25, 30, and 35. \textsc{NextAction} below applies these rankings while excluding already selected slots. The planning loop is bounded by 12 attempts in Text/Math and $4\max(3,k)$ in the agent runner.

\begin{algorithm}[H]
\caption{Trace-guided refresh without an immediate candidate evaluation}
\label{alg:app-refresh}
\small
\begin{algorithmic}[1]
\Require Proposer $G$, encoder $E$, interfaces $\mathcal I$; two children per refresh in distinct slots
\Function{RefreshPools}{$t,h_t,h^\star,\mathcal P,\mathcal V,\mathcal A,\mathcal U$}
    \State $\mathcal Z\gets$ recurring failure groups and representative validation traces from $\mathcal A$
    \State $\Pi\gets[\,]$; $J\gets\varnothing$
    \State $a\gets\Call{NextAction}{\mathrm{failure},\mathcal Z,h_t,h^\star,\mathcal P,\mathcal U,J}$
    \If{$a\neq\bot$}
        \State Append $a$ to $\Pi$; add its slot to $J$
    \EndIf
    \State $\omega\gets$ the rotating default mode order for this refresh
    \For{$i=1,\ldots,R_{\mathrm{plan}}$}
        \State \textbf{if} $|\Pi|=2$ \textbf{then break}
        \State $a\gets\Call{NextAction}{\omega_i,\mathcal Z,h_t,h^\star,\mathcal P,\mathcal U,J}$
        \If{$a\neq\bot$}
            \State Append $a$ to $\Pi$; add its slot to $J$
        \EndIf
    \EndFor
    \State Require $|\Pi|=2$; otherwise report an incomplete refresh plan
    \If{some planned slot has no old module outside $h^\star$ that can be evicted}
        \State \Return $(\mathcal P,\mathcal V,\mathcal U)$ \Comment{Keep the fixed-capacity pools unchanged}
    \EndIf
    \State $\mathcal C_{\mathrm{new}}\gets\varnothing$
    \ForAll{actions $a=(j,c_{\mathrm{parent}},\mathrm{mode})\in\Pi$}
        \State $Z_j\gets$ slot contract, parent code, failure evidence, and bounded related history
        \State $c'_j\gets\Call{CheckedProposal}{G,I_j,Z_j,\mathcal V_j,\mathrm{mode}}$ \Comment{Algorithm~\ref{alg:app-checks}}
        \State Cache $e(c'_j)$; set its parent and family; initialize $n(c'_j)=s(c'_j)=0$
        \State $\mathcal C_{\mathrm{new}}\gets\mathcal C_{\mathrm{new}}\cup\{(j,c'_j)\}$
    \EndFor
    \ForAll{$(j,c'_j)\in\mathcal C_{\mathrm{new}}$}
        \State $P_j^+\gets P_j\cup\{c'_j\}$; $\mathcal D_j\gets\{c\in P_j:c\notin h^\star\}$
        \State $c_j^-\gets\operatorname*{arg\,min}_{c\in\mathcal D_j}\bigl(v(c),\delta(c;P_j^+),\operatorname{id}(c)\bigr)$
        \State $P_j\gets P_j^+\setminus\{c_j^-\}$; retain both implementations in $\mathcal V_j$
        \State Mark $c'_j$ active and $c_j^-$ retired; record the proposal and replacement in $\mathcal U$
    \EndFor
    \State Require $|P_j|=N$ for every slot; commit the updated pool and history records
    \State \Return $(\mathcal P,\mathcal V,\mathcal U)$ \Comment{$A$, $b$, and evaluated keys are unchanged}
\EndFunction
\end{algorithmic}
\end{algorithm}

The retention and redundancy quantities used in the lexicographic eviction rule are
\begin{equation}
 v(c)=\begin{cases}s(c)/n(c),&n(c)>0,\\0,&n(c)=0,\end{cases}
 \qquad
 \delta(c;P)=\min_{v\in P\setminus\{c\}}\bigl(1-e(c)^\top e(v)\bigr).
 \label{eq:app-retention}
\end{equation}
Lower observed mean reward is removed first; smaller nearest-neighbor distance breaks a tie. The incumbent and new child are protected during replacement. Retired code remains archived; later selection uses the ordinary acquisition, without forced child evaluations.

\clearpage
\subsection{Admissibility, history construction, and frozen execution}
\label{app:algorithm_contracts}

Initialization and refresh share the same admissibility principle: a candidate must implement its assigned responsibility and be executable through the declared interface. \textsc{Check} in Algorithm~\ref{alg:app-checks} combines syntax and signature validation, the permitted-capability checks, task-specific return-shape and state checks, and small execution fixtures. It also rejects duplicate identifiers and duplicate implementations within the slot's historical catalog. Rejections are returned to the proposer as concrete correction feedback. The default limit is 20 attempts per requested module for both initialization and refresh.

\begin{algorithm}[H]
\caption{Checked module generation and isolated evaluation of the frozen harness}
\label{alg:app-checks}
\small
\begin{algorithmic}[1]
\Function{CheckedProposal}{$G,I_j,Z,\mathcal V_j,\mathrm{mode}$}
    \State $F\gets[\,]$ \Comment{Rejected code and admissibility feedback}
    \For{$a=1,\ldots,20$}
        \State $c\gets\Call{Propose}{G,I_j,Z,\mathrm{mode},F}$
        \State $(\mathrm{ok},\mathrm{reason})\gets\Call{Check}{c,I_j,\mathcal V_j}$
        \If{$\mathrm{ok}$}
            \State \Return $c$
        \EndIf
        \State Append the rejected implementation and reason to $F$
    \EndFor
    \State Report generation exhaustion; do not commit an unchecked module or an undersized pool
\EndFunction
\Statex
\Function{EvaluateFrozen}{$H^\star,D_{\mathrm{test}},D_{\mathrm{tr}},\mathcal C$}
    \For{$r=1,2,3$}
        \State Initialize an independent test repetition $r$
        \State Reset task memory and environment episodes; initialize permitted training memory
        \State $S^{(r)}\gets\Call{TrustedTest}{H^\star,D_{\mathrm{test}},D_{\mathrm{tr}},\mathcal C}$
        \State Store predictions, trusted scores, and resource use; do not update the optimizer
    \EndFor
    \State \Return $\frac13\sum_{r=1}^3 S^{(r)}$
\EndFunction
\end{algorithmic}
\end{algorithm}

\paragraph{History is selective context, not a new retrieval corpus.}
Failure statistics are accumulated over the committed validation history. The proposer receives a bounded view: up to eight error groups, at most two representative examples per group, up to six related family edits, and up to three higher-scoring historical trajectories for currently failed problems. Contrasts are chosen preferentially from configurations differing in fewer slots. The additional history context is limited to 24,000 characters by dropping complete records. Each record retains its originating evaluation and component identities, so the proposer can distinguish an observed result from the stated intent of an untested edit. These contrasts guide a modification hypothesis; the next complete-harness evaluation supplies its observed reward.

\paragraph{Three state transitions have different roles.}
A completed evaluation changes $(A,b)$, the archive, and usage statistics. A refresh changes which programs are available while retaining those statistics. Finalization freezes a validation-selected composition and removes the optimizer from the execution path. In particular, the empirical mean $v(c)$ is used for pool retention, whereas the complete-harness reward $r_t$ is the observation for the bandit. A new module starts with zero observations and is represented immediately in the same feature space. This is the operational counterpart of Proposition~\ref{prop:app-transfer}: reuse does not require assigning an inherited or synthetic reward to a new implementation.

\paragraph{Evaluation budget.}
The default initialization creates 70 modules for the seven-slot tasks or 50 for ScienceWorld. Seven refreshes can add at most 14 modules while preserving an active space of $10^7$ or $10^5$ compositions. The budget contains 40 complete-composition evaluations; generation, embedding, and checks are accounted for separately. Final testing follows Algorithm~\ref{alg:app-checks}.

\clearpage
\section{Implementation Details}
\label{app:implementation}

This section describes the datasets, runtime settings, and evaluation protocol used in Section~\ref{sec:experimental_setup}. MESH-Harness and Meta-Harness share the task data, execution model, and scoring rules. Their search budgets count complete candidate evaluations; each candidate may make several model calls while solving the task.

\subsection{Datasets and retrieval resources}
\label{app:data}

Table~\ref{tab:app-splits} summarizes the task splits. Training data provide examples or initial memory, validation data guide harness optimization and selection, and held-out data measure the selected harness. The splits remain fixed across methods and repeated test runs.

\begin{table}[ht]
\centering
\caption{Task splits and prediction units. A dash indicates no separate training split. Math contains 200 validation problems from each source; its 500 test problems comprise 254 from OlympiadBench and 246 from MathArena. ScienceWorld counts are episodes.}
\label{tab:app-splits}
\small
\setlength{\tabcolsep}{5pt}
\begin{tabular}{@{}lrrrp{5.0cm}@{}}
\toprule
Dataset & Train & Validation & Test & Prediction or interaction unit \\
\midrule
USPTO & 50 & 30 & 100 & Precursor strings for a product molecule \\
Symptom-to-Diagnosis & 200 & 50 & 212 & A diagnosis from a symptom description \\
LawBench & 200 & 50 & 100 & A set of charges from a case description \\
Mathematical reasoning & -- & 400 & 500 & A final mathematical answer \\
LiveCodeBench & 100 & 100 & 100 & One selected Python program per problem \\
ScienceWorld & 66 & 22 & 102 & One interactive episode \\
\bottomrule
\end{tabular}
\end{table}

\paragraph{Text tasks.}
The three text datasets require different outputs. USPTO asks for reactant precursors given a product molecule~\citep{schneider2016whats}; Symptom-to-Diagnosis asks for a diagnosis from symptoms~\citep{gretelai2023symptom}; and LawBench asks for the applicable charges in a legal case~\citep{fei2024lawbench}. For symptom diagnosis, we use Gretel's 22-class adaptation of the original Symptom2Disease dataset, which filters diagnosis categories and rewrites symptom descriptions with an LLM. We use the same task-specific output and scoring rules for all harnesses.

\paragraph{Mathematical reasoning.}
Math combines answer-based problems from OlympiadBench~\citep{he2024olympiadbench} and MathArena~\citep{dekoninck2026matharena}. The development set contains 200 problems from each source, totaling 400 problems, and is used for harness optimization. The held-out test set contains 254 OlympiadBench problems and 246 MathArena problems, totaling 500 problems. Final evaluation follows the common held-out protocol used for Meta-Harness: the two sources are pooled, and every test instance receives equal weight. Section~\ref{app:evaluation} defines this aggregation. The development and test questions are disjoint, and answers are evaluated by the corresponding source-specific programmatic graders.

All retrieval-capable math harnesses use a fixed corpus of 433,926 problem--solution pairs drawn from NuminaMath~\citep{li2024numinamath}, Omni-MATH~\citep{gao2024omnimath}, DeepMath~\citep{he2025deepmath}, OpenMathReasoning~\citep{moshkov2025aimo2}, and PolyMath~\citep{aimocorpus2026polymath}. Table~\ref{tab:app-math-composition} gives its composition. The corpus is filtered against the evaluation questions and deduplicated before search; the same corpus is used for validation and testing. Retrieval indexes the full corpus. Each solution is truncated to at most 4,000 characters during preprocessing and 3,000 characters when inserted into a prompt.

\begin{table}[ht]
\centering
\caption{Math retrieval corpus after evaluation-overlap filtering and deduplication.}
\label{tab:app-math-composition}
\small
\begin{tabular}{@{}lr@{}}
\toprule
Retrieval source & Problem--solution pairs \\
\midrule
NuminaMath & 55,968 \\
Omni-MATH & 3,201 \\
DeepMath & 98,670 \\
OpenMathReasoning & 267,647 \\
PolyMath & 8,440 \\
\midrule
Total & 433,926 \\
\bottomrule
\end{tabular}
\end{table}

\paragraph{Code generation.}
We use fixed training, validation, and test splits from LiveCodeBench release v6~\citep{jain2025livecodebench}, with 100 problems in each split. The test set contains 23 easy, 22 medium, and 55 hard problems. A harness can use public examples to test and repair its program; private tests are reserved for final scoring. The cap-8 protocol is defined in Section~\ref{app:cap8}.

\paragraph{Interactive scientific tasks.}
ScienceWorld~\citep{wang2022scienceworld} uses the episode splits provided with ADIAS~\citep{jiang2026adias}, covering 22 task types. Each type contributes three training variations and one validation variation, with 102 held-out episodes in total. Episodes use the easy simplifications, a limit of 30 agent turns, and an environment step limit of 200. During search, each iteration samples 15 training episodes without replacement and evaluates all 22 validation episodes.

\subsection{Memory and information available to each harness}
\label{app:boundaries}

The optimizer's validation archive, a harness's runtime memory, and the fixed retrieval corpus serve different roles. Validation scores and traces guide code evolution. Runtime memory supports the current task, while retrieval supplies external examples. Held-out answers remain with the evaluator throughout execution.

\paragraph{Text.}
Each candidate builds a separate memory for each dataset through one pass over its training examples, receiving the gold label after each training prediction. Validation and test predictions use this training-derived memory without adding held-out labels. Memory is rebuilt for each candidate and each final repetition.

\paragraph{Math.}
The harness receives the problem, its answer-format instructions, and access to the fixed retrieval corpus. The evaluator retains the reference answer and grader. Retrieved solutions belong to the external corpus; development and test solutions are not added to it. A candidate's verifier can inspect and refine an answer using the available context, but cannot consult the reference answer.

\paragraph{LiveCodeBench and ScienceWorld.}
Code-generation harnesses receive problem statements, starter-code interfaces, and public examples. Only public-test results can be used for repair. ScienceWorld harnesses receive the task description and observations from their own actions, with memory reset at the start of every episode. Validation traces from these interactions can guide subsequent module proposals, as specified in Section~\ref{app:algorithm_contracts}.

\subsection{Search settings and execution limits}
\label{app:settings}

Text, Math, and LiveCodeBench have seven functional slots; ScienceWorld has five. Each slot starts with ten LLM-generated implementations. Table~\ref{tab:app-settings} lists the shared search parameters. The module encoder is Qwen3-Embedding-8B~\citep{zhang2025qwen3embedding}; each slot is represented by a normalized 128-dimensional vector, cached throughout the run. Newly proposed modules use the same encoder and dimensionality reduction as the initial pool.

\begin{table}[ht]
\centering
\caption{Default search settings. The window $W$ and refresh interval are measured in completed candidate evaluations.}
\label{tab:app-settings}
\small
\setlength{\tabcolsep}{5pt}
\begin{tabular}{@{}lp{7.3cm}@{}}
\toprule
Setting & Value \\
\midrule
Candidate-evaluation budget & 40 \\
Implementations per slot & 10 \\
Feature dimensions & $d=128$ per slot; $p=896$ for seven slots, $p=640$ for five \\
Ridge regularization; UCB coefficient & $\lambda=1$; $\alpha=0.8$ \\
Coverage; family-reuse penalty & $\eta=0.04$; $\gamma=0.04$ \\
Similarity attenuation; minimum scale & $\beta=0.25$; $g_{\min}=0.5$ \\
Recent family-usage window & $W=10$ \\
Start probabilities & Incumbent 0.65, archived 0.20, random 0.15 \\
Archived start candidates & Top $K=5$ configurations \\
Coordinate sweeps & 2 per selection \\
Novelty sampling; module proposal limit & 200 random attempts; 20 proposals per requested module \\
Pool refresh & Every 5 evaluations, while budget remains \\
Children per refresh & At most 2, in distinct slots \\
Search seed & 42 \\
\bottomrule
\end{tabular}
\end{table}

\paragraph{Models and token budgets.}
Both optimizers use Claude Code with Claude Opus 5 at medium effort for proposals and DeepSeek V4 Flash with thinking disabled for execution, as in Section~\ref{sec:experimental_setup}. Maximum output lengths per executor call are 512 tokens for Text, 2,400 for Math, 8,000 for LiveCodeBench, and 512 for ScienceWorld. Text and ScienceWorld use temperature zero; the selected math harness uses temperature 0.1.

\paragraph{Task-specific computation.}
Math permits one primary solver call and at most one verifier call per problem. LiveCodeBench allocates separate draft and repair calls and skips repair when a draft passes every public example. ScienceWorld permits at most one solver call per agent turn. Retrieval, output parsing, and program tests are local operations. Executor identity and scoring remain fixed while the optimizer changes module implementations.

\subsection{Metrics and aggregation}
\label{app:evaluation}

Scores are reported on a 0--100 scale. Each selected harness is frozen and evaluated three times on the same held-out data. For task score $S^{(r)}$ in repetition $r$, the reported score is $S=\frac13\sum_{r=1}^{3}S^{(r)}$. The cross-task score in Figure~\ref{fig:overall_results} is the unweighted mean of the four task scores.

\paragraph{Text.}
Let $s_d^{(r)}$ be accuracy on dataset $d$ in repetition $r$. Text gives equal weight to the three datasets:
\begin{equation}
 S_{\mathrm{Text}}^{(r)}=
 \frac{s_{\mathrm{USPTO}}^{(r)}+s_{\mathrm{Symptom}}^{(r)}+s_{\mathrm{Law}}^{(r)}}{3}.
 \label{eq:app-text-metric}
\end{equation}
USPTO uses exact set agreement on normalized precursor strings; Symptom-to-Diagnosis uses exact diagnosis agreement after case, whitespace, and punctuation normalization; LawBench requires agreement on the full set of charges. The same dataset-macro convention defines the validation reward.

\paragraph{Math.}
We use the pooled pass@1 protocol shared with Meta-Harness. Let $n_{\mathrm O}=254$ and $n_{\mathrm A}=246$ be the numbers of test instances from OlympiadBench and MathArena, and let $C_{\mathrm O}^{(r)}$ and $C_{\mathrm A}^{(r)}$ be their correct-answer counts in repetition $r$. Then
\begin{equation}
 S_{\mathrm{Math}}^{(r)}=
 100\,\frac{C_{\mathrm O}^{(r)}+C_{\mathrm A}^{(r)}}{n_{\mathrm O}+n_{\mathrm A}}.
 \label{eq:math-aggregation}
\end{equation}
Every instance has equal weight, so the two source-level percentages are weighted by their test-set sizes. Each source's programmatic grader evaluates the final extracted answer. Internal retrieval or verification calls are part of one harness execution; the resulting final answer is the prediction scored for that instance.

\paragraph{LiveCodeBench: cap-8 solve rate.}
\phantomsection
\label{app:cap8}
For each problem, the evaluator concatenates the public and private test cases, in that order, and retains at most eight. The selected program solves the problem if it passes every retained case within the six-second timeout per case. The solve rate is the percentage of test problems solved. A harness may generate several drafts or use public feedback for repair, but only its final selected program is scored. The cap refers to test cases per problem.

\paragraph{ScienceWorld.}
For episode $i$, let $s_{i,t}^{(r)}$ be the native environment score at step $t$. The episode contributes its highest attained progress, clipped to the 0--100 range:
\begin{equation}
 z_i^{(r)}=\min\!\left(100,\max\!\left(0,\max_t s_{i,t}^{(r)}\right)\right),
 \qquad
 S_{\mathrm{SW}}^{(r)}=\frac{1}{102}\sum_{i=1}^{102}z_i^{(r)}.
 \label{eq:app-sw-metric}
\end{equation}
Partial progress therefore contributes to the score, while full episode completion corresponds to 100.

\subsection{Final selection and cost accounting}
\label{app:selection}

\paragraph{Validation-based selection.}
The final harness is selected from the validation archive. Math ranks candidates by development pass@1, breaking ties by lower mean input-token use and then by candidate identifier. LiveCodeBench and ScienceWorld use validation score, then lower total token use and candidate identifier. Text uses validation accuracy under the strict component interface. The selected code, executor settings, and retrieval resources are fixed before testing. Each test repetition starts with fresh task memory and environment episodes.

\paragraph{Search budget.}
MESH-Harness evaluates 40 complete compositions; Meta-Harness evaluates two candidates in each of 20 proposal rounds~\citep{lee2026metaharness}. Pool initialization, embedding, interface checks, and coordinate comparisons are optimization work. Reference-harness evaluations are recorded separately from the 40 new-candidate evaluations. All initialization and reference-execution expenses are included in search cost.

\paragraph{Reported cost.}
The test-only totals in Table~\ref{tab:search_execution_cost} measure execution of the frozen harnesses. Search-plus-test totals additionally include proposer, executor, and embedding usage during optimization, with provider-reported cache accounting. At deployment, only the selected harness and its task resources are required. Infrastructure interruptions are resumed as the same evaluation; an executed answer or program that fails the task receives the task's failure score.

\clearpage
\section{Analysis of Optimized Harnesses}
\label{app:harness_analysis}

We analyze the final harnesses selected by validation performance on the four tasks. Their structures show how MESH turns a pool of reusable implementations into task-specific programs: modules determine what information to retain, how to present it to the model, and how to act on the resulting response. We first examine these decisions within each task, then relate the selected components to their origins in the pool. The boxes reproduce prompt or code excerpts from the final implementations, with only line breaks or outer indentation adjusted for display.

\subsection{Text: evidence composition with explicit answer contracts}
\label{app:harness-text}

The Text harness separates seven decisions: maintaining training examples, defining the answer space, retrieving evidence, composing context, allocating computation, checking the response, and extracting the final answer. A shared answer contract connects these decisions across single-label prediction, multi-label prediction, and structured-string generation. The budget policy identifies the requested answer type from the input and stored examples, allowing downstream components to preserve the semantics of each task. Structured strings, for example, remain open-ended outputs rather than entries in a finite label inventory.

The selected memory module removes exact input--target duplicates and, when its pool exceeds capacity, preferentially removes redundant examples from large signature groups. The retriever ranks examples using weighted feature overlap, input-length proximity, and script-dependent features, while suppressing near-duplicates and repeated answer labels. The composer displays at most four supplied examples and adds a short ``Decisive difference'' annotation to each: features shared with the query but absent from the other displayed cases are identified explicitly. Retrieval and composition thus serve complementary roles: one selects examples, while the other exposes the distinctions the model should consider.

\begin{meshpanel}{Text: explicit label-set reasoning}
Examine each actor and each act on its own and decide
which labels it supports.
Report every applicable label. A combination never seen
together in the cases below is allowed as long as each
label is a known atomic label.
Before answering, re-scan the input for a supported label
you left out, then drop any label with no direct support
in the input.
\end{meshpanel}

The answer contract continues through generation and parsing. A JSON field carries the final answer, with separate rules for a complete label, a combination of labels, or a structured string. For the latter, the prompt requests preservation of stereochemistry, charges, and ring closures. The optimized program therefore joins evidence selection, comparison, and answer formatting through explicit interfaces. These connected choices make complete-program evaluation relevant: an example is useful to the solver only through the context and decision rules with which it is presented.

\subsection{Math: compact problem context and exact-answer discipline}
\label{app:harness-math}

The Math harness contains seven slots for query analysis, budget allocation, retrieval, context composition, solving, verification, and answer parsing. The selected combination centers on the requested mathematical quantity and its output form. Its analyzer distinguishes counting, finding all solutions, extremal questions, and exact-value questions, then extracts constraints and corresponding obligations. These include completeness for a solution set, or an attaining construction and a matching bound for an extremum.

The selected composer preserves the complete problem and adds compact planning notes and an exact-answer contract. It does not insert the retrieved passages, even though retrieval is available through the task interface. This is a concrete choice within the optimization space: a component may determine which available information enters the solver context. Here, the resulting prompt keeps the problem, its constraints, and the required form of the answer together.

\begin{meshpanel}{Math: compact context and exact answers}
Solve one olympiad math problem. No reference material
is supplied.

EXACT FORM: keep radicals, fractions, factorials and
symbolic constants unevaluated rather than decimal; use
set or interval notation for collections or ranges;
omit units unless the problem asks for them.

Run the check, then emit the complete JSON final_answer;
if later checking corrects it, emit a new complete JSON
object with the corrected value.
\end{meshpanel}

The solver adds a problem-dependent instruction: maintain a ledger when many named quantities occur, substitute explicitly given numerical constants when applicable, or proceed through the governing identity. For long derivations it requests forward progress without repeated restarts and a complete answer once a defensible value is available. The verifier checks specific completion defects, including a missing answer after truncation and an unresolved symbol where a definite result is requested. The parser first uses explicit answer fields, then applies answer-region and exact-form fallback rules. The resulting division of labor carries the answer specification from problem analysis through derivation and final extraction, while preserving a compact solver context.

\subsection{LiveCodeBench: complementary drafts and failure-directed repair}
\label{app:harness-lcb}

The LiveCodeBench harness separates problem analysis, budget selection, prompt construction, draft generation, public-test feedback, repair selection, and code parsing. The analyzer extracts the entry-point contract, requested output, input-size bounds, and boundary conditions. The prompt distinguishes a function or class method from a standard-input program and reproduces public examples. The selected budget requests two drafts with repair enabled. Both drafts use temperature zero, with different instructions: one models the process directly; the other is asked to re-derive the requested output quantity independently, using requirement sentences extracted from the statement when available.

Each draft is evaluated on the public tests. The repair selector compares failures across drafts: a failing row is \emph{local} if another draft passes it, \emph{contested} if the drafts produce different wrong outputs, and \emph{shared} otherwise. Syntax, timeout, and runtime defects take precedence when choosing the repair instruction.

\begin{meshpanel}{LiveCodeBench: public-feedback arbitration}
mine = _clip(r.get('actual'), 300)
tag = 'shared'
for o in others:
    c = _cell(o, i)
    if c is None:
        continue
    if c.get('passed') is True:
        tag = 'local'
        break
    if _clip(c.get('actual'), 300) != mine:
        tag = 'contested'
\end{meshpanel}

The feedback interface retains the evidence needed for this comparison, including outputs and failures for individual public-test rows. The repair module can therefore use a passing alternative to guide a local correction, or direct attention to the required output when multiple drafts disagree or fail together. Its prompt includes a small set of failing rows, requests preservation of passing behavior, and warns against special-casing public examples. This case illustrates how reusable components can exchange structured execution evidence: draft generation supplies alternatives, feedback characterizes them, and repair turns those observations into a targeted instruction.

\subsection{ScienceWorld: planning with observed preconditions}
\label{app:harness-sw}

The ScienceWorld harness uses five slots: evidence memory, goal-phase planning, action-candidate ranking, a focus guard, and recovery. Evidence memory records attempted actions, observed changes, score increments, locations, and focus history. The planner converts this record into a concrete next subgoal; ranking and guard modules constrain which actions are proposed and accepted. Recovery uses the same state and permission interfaces to replace an invalid or repeatedly ineffective action.

The selected planner also illustrates how a new implementation is incorporated into the pool. It was introduced at search iteration 24 by a trace-aware, failure-driven refresh targeting incomplete goals. The planner constructs an ordered chain of preconditions: perform a permitted focus action, navigate toward an absent target, inspect a visible object's exact identity, establish a missing circuit element, measure an unresolved property, or execute a report branch supported by an observation. The following excerpt distinguishes an unknown property from a satisfied report condition.

\begin{meshpanel}{ScienceWorld: observation-dependent planning}
try:
    val = observed_property(r['condition'], st, bd)
except Exception:
    val = None
if val is None:
    if pending is None:
        pending = r
elif val is True and ready is None:
    ready = r
\end{meshpanel}

When progress stalls, the planner can move to a less-attempted later precondition from its chain. The focus guard and recovery module then check proposed actions through their own interfaces. The refreshed planner was selected at iterations 27, 29, 32, 35, and 36, and the final iteration-36 combination joins it with four components from the initial pool. A new planning policy can therefore be evaluated and reused while the surrounding evidence, ranking, guard, and recovery implementations remain available for recombination.

\subsection{Reuse, selection, and evolution in the final programs}
\label{app:harness-artifacts}

The final Text, Math, and LiveCodeBench harnesses each combine seven implementations from the initial pool; ScienceWorld combines four initial implementations with the refreshed planner. Their provenance illustrates the two roles of the pool: preserving alternatives for combination selection and admitting new implementations during search. Retention and evolution operate at the level of individual slots, while evaluation measures the behavior of the assembled program.

Across tasks, the selected programs organize different forms of evidence: comparisons among examples in Text, mathematical constraints and answer forms in Math, public-test outcomes in LiveCodeBench, and observed preconditions in ScienceWorld. Stable interfaces allow these control policies to be stored and recombined without prescribing a single reasoning workflow. Together with the main-text ablations, these cases connect the search design to the concrete programs it produces.

\end{document}